\documentclass{article}
\usepackage{enumitem}
\usepackage{arxiv}

\usepackage[utf8]{inputenc} 
\usepackage[T1]{fontenc}    
\usepackage{booktabs}       
\usepackage{nicefrac}       
\usepackage{microtype}      
\usepackage{lipsum}
\usepackage{multirow}
\usepackage[table]{xcolor}
\usepackage{arydshln}
\usepackage{siunitx}
\usepackage{makecell} 
\usepackage{mathrsfs}
\usepackage{svg}
\usepackage{algpseudocode}
\usepackage{algorithm}
\usepackage{array}
\usepackage[caption=false,font=normalsize,labelfont=sf,textfont=sf]{subfig}
\usepackage{textcomp}
\usepackage{stfloats}
\usepackage{url}
\usepackage{verbatim}
\usepackage{rotating}
\usepackage{graphicx}
\usepackage{multicol}
\usepackage{cellspace}
\usepackage{subcaption}
\usepackage{amsmath, amsfonts, amsthm, amssymb}
\usepackage{fancyhdr}
\usepackage{hyperref}
\usepackage{wrapfig}

\title{Paper08-IEEE}
\author{itxwaleedrazzaq }
\date{}

\newtheorem{theorem}{Theorem}
\newtheorem{lemma}{Lemma}
\newtheorem{corollary}{Corollary}

\theoremstyle{definition}

\newcounter{problem}

\theoremstyle{remark}

\title{FLUID: Continuous-Time Hyperconnected Sparse Transformer for Sink-Free Learning}

\author{
  Waleed Razzaq \\
  School of Automation\\
  University of Science and Technology China\\
  Hefei, Anhui \\
  \texttt{waleedrazzaq@mail.ustc.edu.cn} \\
  \And
  Yun-Bo Zhao\thanks{Corresponding author. Email: \texttt{ybzhao@ustc.edu.cn}} \\
  School of Automation\\
  University of Science and Technology China\\
  Hefei, Anhui \\
  \texttt{ybzhao@ustc.edu.cn} \\
}

\begin{document}
\maketitle

\begin{abstract}
Continuous-time (CT) Transformers improve irregular and long-range modeling over CT-RNNs by exploiting inputs or outputs embeddings with continuous dynamics. However, the core scaled-dot-product-attention (SDPA) mechanism remains inherently discrete. We propose FLUID (Flexible Unified Information Dynamics), a CT Transformer that incorporates continuous dynamics directly into the attention computation by replacing it with Liquid Attention Network (LAN). LAN reinterprets attention logits as continuous dynamical system and reformulates them as the solution to a linear ODE modulated by input-dependent nonlinear recurrent gates. Theoretically, we establish stability guarantees for LAN dynamics and show that it serves as an interpolating middle ground between SDPA and CT-RNNs, recovering each as special case under well-defined parameterization of its gating functions. LAN also introduces an explicit attention-sink gate to eliminate disproportionate attention mass on uninformative nodes. FLUID replaces standard residual connections with input-dependent Liquid Hyper-Connections to adaptively regulate interlayer information flow. Empirically, we evaluate FLUID on a broad set of learning tasks, including (i) irregular time-series, (ii) long-range modeling, (iii) lane-keeping control of autonomous vehicles, and (iv) learning physical dynamics under a scarce data regime. Across all the tasks, FLUID consistently matches or outperforms CT baselines, achieving improvements of up to 47\% in certain scenarios and enhancing generalization under distributional shifts. Additionally, FLUID demonstrates superior noise robustness and a self-correcting inductive bias in autonomous vehicle control. We also provide a detailed analysis of key hyperparameters to guide tuning and show that FLUID occupies an intermediate position among competing approaches in terms of runtime and memory efficiency.
\end{abstract}


\section{Introduction}

Sequential modeling is a fundamental paradigm in AI research for real-world applications such as robotics~\cite{yu2025recent, razzaq2023neural}, medicine~\cite{rostami2025hierarchical}, and industry~\cite{oikonomou2025time, razzaq2025carle}. Traditionally, discrete-time (DT) models such as RNN~\cite{rumelhart1985learning}, LSTM~\cite{hochreiter1997long}, and GRU~\cite{cho2014learning} have been the default choice for sequence modeling. However, real-world sequence data are often irregularly sampled or span long time horizons, which limits the expressive power of these models, which may overlook important fine-grained information~\cite{chen2019neuralordinarydifferentialequations,rubanova2019latent}. Continuous-time (CT) learning offers a promising alternative by using the mathematics of differential equations to model such fine-grained dynamics. It also provides inherent advantages in terms of stability and interpretability. Neural ODEs~\cite{chen2019neuralordinarydifferentialequations} made CT learning practical by parameterizing dynamics with neural networks solved via adaptive ODE solvers~\cite{dormand1980runge,hindmarsh2005sundials}, yet they incur high computational cost, struggle with discontinuous dynamics, and lack the expressiveness needed for complex sequential tasks. CT-RNNs~\cite{rubanova2019latent} introduced a time-constant factor to stabilize Neural ODE training but remain sensitive to hyperparameter choices and can degrade over long horizons. Mixed-memory RNNs~\cite{lechner2022mixed} address these issues through dual-memory pathways that separately track fast fluctuations and slow trends, reducing initialization sensitivity. Liquid Neural Network (LNN)~\cite{hasani2021liquid,hasani2022closed} take biological inspiration to model CT learning behavior as input-dependent time-constant parameters to adapt in real-time. However, this coupling substantially increases training cost, often outweighing the gains in expressivity.

Transformers~\cite{vaswani2017attention} have substantially advanced sequence modeling by replacing recurrence with the scaled dot-product attention (SDPA) mechanism, which enables parallel computation and effectively captures long-range dependencies by emphasizing the most relevant elements in a sequence. However, SDPA Transformers are inherently discrete, as the SDPA is computed over indexed token positions with attention weights defined on discrete sequence elements. Consequently, they lack the native inductive bias for representing CT trajectories or capturing the smoothly evolving dynamics essential for real-world systems. This limitation motivates the integration of CT dynamics into Transformers to better model smoothly varying temporal processes while retaining strength in long-range dependency modeling.

Recent work has attempted to augment Transformers with CT dynamics. mTAN~\cite{shukla2021multi} introduces time-aware attention with continuous embeddings that interpolate observations into fixed-length representations, improving robustness. However, the reliance on interpolation may obscure the fine temporal structure. Continuous-time attention (CTA)~\cite{chien2021continuous} generalizes attention by parameterizing hidden states and attention mechanisms using Neural ODEs, enabling modeling of irregular sequences at high computational cost. ODEFormer~\cite{d2023odeformer} learns underlying differential equations from irregular trajectories, improving interpretability but relying heavily on strong data assumptions and often synthetic training settings. ContiFormer~\cite{chen2023contiformer} combines latent ODE dynamics with time-aware attention to jointly model temporal evolution dependencies, still incurring high computational cost.

Despite their individual limitations, CT Transformers share several common issues: (i) existing methods typically apply CT dynamics to either input or output embeddings, while the underlying attention is still computed over discrete samples rather than a truly continuous temporal domain; (ii) they are susceptible to attention sinks~\cite{xiao2023efficient}, where disproportionate attention mass is assigned to initial or uninformative tokens; (iii) they use residual connections that can induce the seesaw effect~\cite{zhu2024hyper}, where the model struggles to balance new information with the historical context; (iv) they often require large amounts of refined data to learn meaningful patterns; and (v) they exhibit limited robustness under distribution shifts, particularly to unseen sampling frequencies.

We propose \textit{FLUID (Flexible Unified Information Dynamics)}, a CT Transformer architecture that addresses these limitations by incorporating Liquid Attention Network (LAN) as its core attention mechanism. Rather than applying CT dynamics to inputs or outputs, FLUID embeds them directly into the attention logits computation, where logits evolve according to a parameterized linear ODE modulated by input-dependent nonlinear, interlinked gates that couple query–key interactions over time. LAN incorporates an attention-sink gate to mitigate attention-sink behavior. To address the seesaw effect, FLUID replaces standard residual connections with input-dependent Liquid Hyper-Connections~\cite{zhu2024hyper} ($\mathcal{HC_{\text{liquid}}}$) to regulate the historical context through input-dependent gates. We evaluate FLUID across a diverse set of learning tasks: (i) irregular time-series modeling; (ii) long-range sequence modeling; (iii) lane-keeping control for autonomous vehicles; and (iv) learning physical dynamics under scarce data.

Our contributions are as follows:
\begin{itemize}
    \item We propose the FLUID Transformer architecture, which shifts CT dynamics from the inputs/outputs to the attention mechanism itself by introducing Liquid Attention Network (LAN). LAN reformulates attention logits through linear ODE modulated by nonlinear input-dependent gates and incorporates an attention-sink gate to mitigate attention-sink phenomena.
    \item Liquid Hyper-Connections (\(\mathcal{HC}_{\text{liquid}}\)) to replace standard residual connections, which apply input-dependent gating to regulate the balance between new and historical context addressing the seesaw effect.
    \item We conduct extensive evaluations across diverse learning tasks: (i) irregular time-series modeling; (ii) long-range sequence modeling; (iii) lane-keeping control of autonomous vehicles; and (iv) learning physical dynamics under scarce data.
\end{itemize}

The remainder of the paper is organized as follows. Section~\ref{sec:related_work} reviews the relevant literature and highlights how our approach differs from prior work. Section~\ref{sec:fluid} presents the comprehensive architectural details of the proposed FLUID transformer. Section~\ref{sec:eval} comprehensively evaluates FLUID across four diverse sets of learning tasks. Section~\ref{sec:discuss} discusses the possible future extensions of the work. Finally, Section~\ref{sec:conclude} concludes the paper.

\section{Related Work}\label{sec:related_work}
\textbf{Continuous-time RNNs:} CT-RNN variants such as CT-RNN~\cite{rubanova2019latent}, GRU-ODE~\cite{de2019gru} and PhasedLSTM~\cite{neil2016phased} model temporal dynamics through differential equations over hidden states. mmRNNs~\cite{lechner2022mixed} and LNN~\cite{hasani2021liquid, hasani2022closed}, introduce multiple time scales and input-dependent time- constants through nonlinear gating mechanisms. FLUID draws inspiration from these input-dependent mechanisms but applies them to attention computation rather than hidden-state evolution.

\textbf{Continuous-time Transformers:} CT-Transformers extends the SDPA Transformers~\cite{vaswani2023attentionneed} by incorporating either temporal embeddings or differential equations into the model dynamics. mTAN~\cite{shukla2021multi} introduces a time-aware attention mechanism, whereas ContiFormer~\cite{chen2023contiformer} and ODEFormer~\cite{d2023odeformer} parameterize hidden representations using Neural ODEs or latent continuous dynamics. However, in these methods, the core attention computation remains discrete, with continuous modeling applied only to the input or output embeddings of the Transformer. FLUID shifts CT modeling directly into the attention computation itself using LAN, reformulating them as an input-dependent linear ODE modulated with nonlinear, interlinked gates.

\begin{figure*}[t]
\centering
\includegraphics[width=0.95\textwidth]{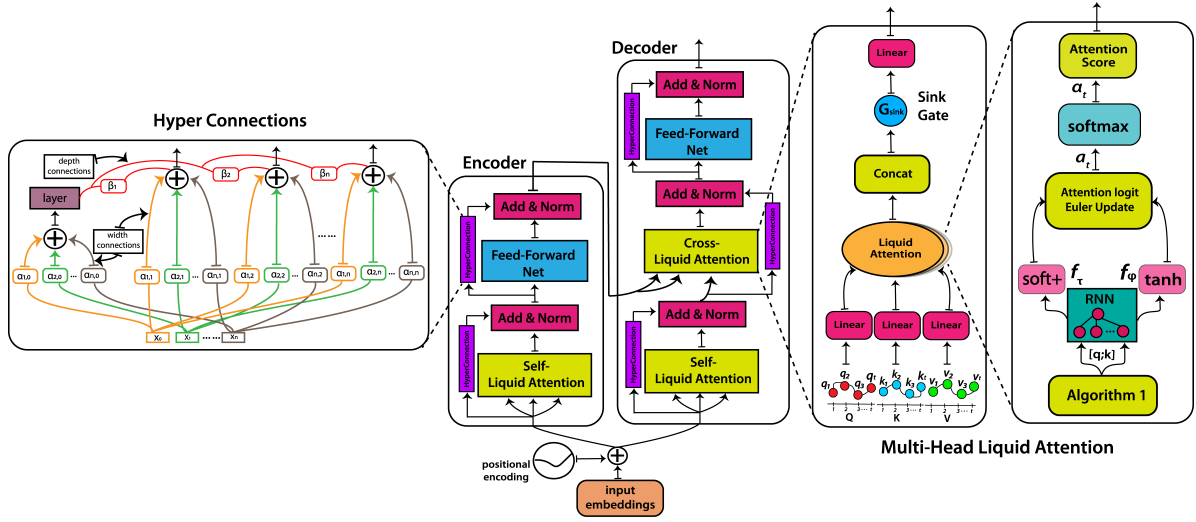}
\caption{Illustration of the internal architecture of FLUID Transformer. Input embeddings are shared between the encoder and decoder. The encoder consists of self-LAN with $G_{\text{sink}}$ and \(\mathcal{HC}_{\text{liquid}}\), while the decoder consists of cross-LAN and self-LAN.}
\label{fig:fluid_architecture}
\end{figure*}

\section{FLUID Transformer}\label{sec:fluid}
In this section, we outline the design of each component of the FLUID Transformer. Similar to the SDPA Transformer, it comprises three components: (i) input embeddings, (ii) encoder, and (iii) decoder. However, the standard SDPA and residual connections are replaced with LAN and \(\mathcal{HC}_{\text{liquid}}\), respectively. For a detailed overview of the SDPA Transformer, refer to Appendix~\ref{appendix:sdpa}. The internal architecture of the FLUID Transformer is illustrated in Figure~\ref{fig:fluid_architecture}.

\textbf{Input Embeddings:} Unlike the SDPA Transformer, where the encoder and decoder use separate embeddings, we share embedding parameters between them. This design enforces consistent temporal representations across both historical and forecasting horizons, improves information transfer via cross-LAN, reduces parameter redundancy, and acts as a regularizer that enhances robustness to noise and data irregularities, leading to more stable performance in low-data and distributional shifts.

\textbf{Encoder:} The encoder block consists of $\mathcal{N}$ stacked layers, each containing a self-LAN module followed by a position-wise feed-forward network (FFN) with $\mathcal{HC_{\text{liquid}}}$ and layer normalization applied after each sublayer. Self-LAN enables each position to dynamically adapt its attention based on the input, allowing flexible temporal dependency modeling. In parallel, $\mathcal{HC_{\text{liquid}}}$ modulates interlayer information flow conditioned on the input, enabling adaptive propagation of representations across depth. For the input ($\mathbf{x}$), the encoder computes:
\begin{align}
\mathbf{h} &= \mathrm{LayerNorm}(\mathcal{HC_{\text{liquid}}}(\mathbf{x} , \mathrm{LAN}(\mathbf{x}, \mathbf{x}, \mathbf{x}))) \\
\mathrm{ENC}(\mathbf{x}) &= \mathrm{LayerNorm}(\mathcal{HC_{\text{liquid}}}(\mathbf{h} , \mathrm{FFN}(\mathbf{h})))
\end{align}

\textbf{Decoder:} The decoder also consists of $\mathcal{N}$ stacked layers, each consisting of three sub-modules: (i) masked self-LAN; (ii) cross-LAN; and (iii) FFN. Masked Self-LAN restricts each position to attend only to previous positions, ensuring autoregressive generation, while cross-LAN conditions on the encoder representations. For input ($\mathbf{x}$) and encoder output ($\mathbf{z}$), decoder computes:
\begin{align}
\mathbf{h_1} &= \mathrm{LayerNorm}(\mathcal{HC_{\text{liquid}}}(\mathbf{x} , \mathrm{LAN}(\mathbf{x}, \mathbf{x}, \mathbf{x}))) \\
\mathbf{h_2} &= \mathrm{LayerNorm}(\mathcal{HC_{\text{liquid}}}(\mathbf{h_1} , \mathrm{LAN}(\mathbf{h_1}, \mathbf{z}, \mathbf{z}))) \\
\mathrm{Dec}(\mathbf{x}, \mathbf{z}) &= \mathrm{LayerNorm}(\mathcal{HC_{\text{liquid}}}(\mathbf{h_2} , \mathrm{FFN}(\mathbf{h_2})))
\end{align}

\subsection{Liquid Attention Network (LAN)}
Loosely motivated by the input-dependent time-constant dynamics of LNNs (refer to Appendix~\ref{appendix:prelim_lnn}), we propose to view the computation of attention logits (refer to Appendix~\ref{appendix:prelim_attention}) as a CT dynamical process via a linear ODE modulated by input-dependent nonlinear interlinked recurrent gates, defined as follows:
\begin{equation}
\dot{a}_t
= - \underbrace{g_\tau(\mathbf{u}_t,t, \theta_\tau)}_{\text{$f_\tau(\mathbf{u}_t)$}} \, a_t
   + \underbrace{g_\phi(\mathbf{u}_t,t,\theta_\phi)}_{\text{$f_\phi(\mathbf{u}_t)$}}
\label{eq:liquid_ode}
\end{equation}
where \( f_\tau\) is learnable time-constant gate with input \(\mathbf{u_t=[q;k]}\) defined by query-key interactions, \(\theta_\tau\) learnable parameters and \( g_\tau(\cdot) \) denotes the activation function. Similarly, \( f_\phi \) is content-target gate parameterized using \(\theta_\phi\) and \( g_\phi(\cdot) \) activation. time ($t$) indexes latent logit-refinement axis along which each attention logit converges towards its equilibrium value governed by \( f_\tau\) and \( f_\phi \). We refer this formulation as \textit{Liquid Attention Network (LAN)}.

\textbf{LAN forward-pass via ODE solver:} LAN follows the Neural-ODE~\cite{chen2019neuralordinarydifferentialequations} framework, where forward-pass updates are computed using a numerical ODE solver. We use the \textit{explicit Euler} method because of its simplicity and efficiency. Higher-order solvers such as Runge‒Kutta may offer more accuracy, but their computational cost scales poorly with sequence length, making them impractical for long sequences. Let \(\Delta t\) be the step size, with discrete times \(t_n = n \Delta t\) and logit states \(a_n = a(t_n)\). The update is
\begin{equation}
a_{n+1} = a_n + \Delta t (-f_{\tau,n} a_n + f_{\phi,n})
\label{eq:liquid_euler}
\end{equation}

\subsubsection{Theoretical Analysis}
We now analyze the theoretical aspect of LAN dynamics under both continuous dynamics and Euler approximations.

\begin{theorem}[Forward Invariance and Boundedness]\label{theorem:state_stability}
Let the gating functions be bounded such that $0 < \tau_{\min} \le f_\tau$ and $f_\phi \in [\phi_{\min}, \phi_{\max}]$ for all $t$. Define the instantaneous equilibrium bounds:
\begin{equation}
A_{\min} = \inf_t \left( \frac{f_\phi}{f_\tau} \right), \quad A_{\max} = \sup_t \left( \frac{f_\phi}{f_\tau} \right)
\end{equation}
The interval $\mathcal{I} = [A_{\min}, A_{\max}]$ is a forward-invariant set for the dynamics in Eq.~\ref{eq:liquid_ode}. If the initial condition $a_0 \in \mathcal{I}$, then $a_t \in \mathcal{I}$ for all $t \ge 0$. 
\end{theorem}

\begin{proof}
Let the dynamics be $\dot{a}_t = -f_\tau(a_t - A)$, where $A = f_\phi/f_\tau$. Note that by definition, $A_{\min} \le A \le A_{\max}$ for all $t$. We examine the boundary conditions of the interval $\mathcal{I}$:
\begin{itemize}
    \item \emph{Upper Bound:} Consider the case where the state reaches the upper boundary, $a_t = A_{\max}$. The time derivative is:
    \begin{equation}
    \dot{a}_t = -f_\tau (A_{\max} - A)
    \end{equation}
    Since $f_\tau > 0$ and $(A_{\max} - A) \ge 0$, it follows that $\dot{a}_t \le 0$. Thus, the derivative points inward or is zero, preventing the trajectory from crossing above $A_{\max}$.
    
    \item \emph{Lower Bound:} Consider the case where $a_t = A_{\min}$. The time derivative is:
    \begin{equation}
    \dot{a}_t = -f_\tau (A_{\min} - A)
    \end{equation}
    Since $(A_{\min} - A) \le 0$, the term $-(A_{\min} - A) \ge 0$. Given $f_\tau > 0$, it follows that $\dot{a}_t \ge 0$. The derivative points inward, preventing the trajectory from crossing below $A_{\min}$.
\end{itemize}
Therefore, the set $\mathcal{I}$ is forward invariant. The system is forward invariant and uniformly bounded under bounded gating functions, despite time-varying inputs.
\end{proof}

\begin{lemma}[Numerical Stability of Euler Discretization]\label{lemma:euler_stability}
For the discrete update $a_{n+1} = a_n + \Delta t (-f_{\tau,n} a_n + f_{\phi,n})$ to remain numerically stable and non-oscillatory, the step size $\Delta t$ must satisfy:
\begin{equation}
\Delta t \le \frac{1}{\sup_t f_\tau}.
\label{eq:euler_stability}
\end{equation}
\end{lemma}

\begin{proof}
Rearranging the Euler update yields:
\begin{equation}
a_{n+1} = a_n(1 - \Delta t f_{\tau,n}) + \Delta t f_{\tau,n} \left( \frac{f_{\phi,n}}{f_{\tau,n}} \right).
\end{equation}
Let $\alpha_n = \Delta t f_{\tau,n}$ and target $A_n = f_{\phi,n}/f_{\tau,n}$. The equation becomes a convex combination:
\begin{equation}
a_{n+1} = (1 - \alpha_n)a_n + \alpha_n A_n.
\end{equation}
For $a_{n+1}$ to remain a convex average of the previous state $a_n$ and the instantaneous target $A_n$ (thereby preserving the bounds derived in Theorem~\ref{theorem:state_stability}), we require the coefficients to be non-negative:
\begin{equation}
0 \le \alpha_n \le 1 \implies 0 \le \Delta t f_{\tau,n} \le 1.
\end{equation}
Since $f_\tau > 0$, stability requires $\Delta t \le 1/f_{\tau,n}$. To guarantee stability across the entire horizon, we select $\Delta t \le 1 / \sup_t f_\tau$. Violating this condition (stiffness) results in numerical oscillations or divergence.
\end{proof}

\begin{theorem}[LAN as Dynamical Bridge between SDPA and CT-RNNs]\label{theorem:dynamical_bridge}

The LAN occupy an intermediate position between the memory-free SDPA mechanism and the fully stateful CT-RNN. This interpolation is governed exclusively by the effective time constant of the system $\tau_{sys}:= 1/f_{\tau}(\mathbf{u}_t)$, which LAN learns adaptively per query-key pair. Specifically:

\textbf{(i) Attention limit ($\tau_{sys} \to 0)$:} Suppose the content and time-constant gates satisfy the ratio constraint:
\begin{equation}
\frac{f_\phi(\mathbf{u}_t,t,\theta_\phi)}{f_\tau(\mathbf{u}_t,t,\theta_\tau)}   \approx \frac{q^Tk}{\sqrt{d}}
\end{equation}
As $\tau_{sys} \to 0$ (meaning $f_\tau \to \infty$), the ODE becomes singularly perturbed, and the transient dynamics decay instantaneously, snapping the quasi-state $f_\phi/{f_\tau}$. Furthermore, under the maximum stable step size $\Delta t = 1/f_\tau$ (Lemma~\ref{lemma:euler_stability}), the single step Euler update from the initial logit state $a_0 =0$ yields:
\begin{equation}
a_1 = (1- \Delta t f_\tau)a_0 + \Delta t f_\phi = \frac{f_\phi}{f_\tau} \approx \frac{q^Tk}{\sqrt{d}}
\end{equation}
This approximates the SDPA logit. After softmax normalization, this gives LAN(q,k,v)$\to$ SDPA(Q,K,V): attention with no temporal memory, indifferent to the order or timing of observations.

\textbf{(ii) CT-RNN limit ($\tau_{sys}=\tau$):} Let the gating be feedforward (non-recurrent) function and suppose:
\begin{equation}
f_\tau = \frac{1}{\tau},  \qquad f_\phi = \frac{1}{\tau}\cdot\sigma(W_\phi \mathbf{u}_t + b_\phi)    
\end{equation}
for a fixed global time-constant $\tau>0$, bounded nonlinearity $\sigma$, and learnable weights $W_\phi, b_\phi$. Substituting this into the LAN dynamics yields:
\begin{equation}
    \tau \dot{a_t} = -a_t + \sigma(W_\phi \mathbf{u}_t + b_\phi)
\end{equation}
which exactly recovers the leaky-integrator equation of a CT-RNN with hidden states $h_t := a_t$, input $\mathbf{u}_t = [\mathbf{q};\mathbf{k}]$, and fixed time-constant $\tau$.

\textbf{(iii) LAN as Adaptive Middle Ground:} In unrestricted setting, $f_\tau(\mathbf{u}_t)$ is a learned, input-dependent function, inducing a per-token-pair effective constant $\tau_{sys}(\mathbf{u}_t) = 1/f_\tau(\mathbf{u}_t)$ that varies continuously across the sequence. This encodes two quantities that are mutually exclusive in limiting regimes: content alignment (through $f_\phi$, mirroring attention) and temporal persistence (through $f_\tau$, mirroring CT-RNN memory). For token pairs where $f_\tau$ is large, $\tau_{sys} \to 0$ and LAN behavior is attention-like, instantaneously reflecting content relevance. For pairs where $f_\tau$ is small, $\tau_{sys} \to \infty$ and LAN accumulates history in a CT-RNN fashion.
\end{theorem}

\begin{proof}
\textbf{(i)} In continuous limit, the term $\dot{a_t}/f_\tau$ approaches 0, reducing the ODE to the algebraic equation $a_t = f_\phi/f_\tau$. Discretely, from the Euler update with $a_0 = 0$ and $\Delta t = 1/f_\tau$:
\begin{equation}
a_1 = a_0 (1- \Delta t f_\tau) + \Delta t f_\phi = 0.(1-1) + \frac{f_\phi}{f_\tau}  = \frac{f_\phi}{f_\tau}    
\end{equation}
applying the stated constrains give $a_1 \approx q^\top k/ \sqrt{d}$. $\mathrm{softmax}$ over all key positions then approximate SDPA output.

\textbf{(ii)}  Substituting the assumed gate parameterizations into LAN ODE:
\begin{equation}
\dot{a_t} = - \frac{1}{\tau} a_t + \frac{1}{\tau} \sigma( W_\phi \mathbf{u}_t + b_\phi) 
\end{equation}
multiplying by $\tau$ yields $\tau \dot{a_t} = -a_t + \sigma(W_\phi \mathbf{u}_t + b_\phi)$. Setting $h_t := a_t$ recovers the CT-RNN dynamics.

\textbf{(iii)} In general, neither restrictive condition holds. $f_\tau(\mathbf{u}_t)$ is a continuously valued function learned jointly with $f_\phi(\mathbf{u}_t)$, interpolating between the two regimes on a per-token pair basis without committing globally to either.
\end{proof}

\begin{corollary}
Let $\mathscr{F}_{\text{SDPA}}$ and $\mathscr{F}_{\text{CT-RNN}}$ denote the respective function classes of SDPA and CT-RNNs. These classes are incomparable as SDPA cannot represent temporal history, and CT-RNNs cannot selectively suppress memory on a per-token pair basis. The LAN hypothesis space $\mathscr{F}_{\text{LAN}}$ strictly union satisfies of both classes:
:
\begin{equation}
\mathscr{F}_{\text{LAN}} \supsetneq \left\{\mathscr{F}_{\text{CT-RNN}}, \mathscr{F}_{\text{SDPA}} \right\}
\end{equation}
\end{corollary}

\begin{proof}
SDPA logits $q^\top k /\sqrt{d}$ are time-independent; they carry no sequential state and cannot represent any dynamical system with non-trivial transient behavior. Conversely, a CT-RNN with a fixed global time constant $\tau$ cannot zero out its memory selectively for a subset of input pairs. The Theorem~\ref{theorem:dynamical_bridge} demonstrates that LAN recovers both under distinct, well-defined parameterizations of $f_\tau$ and $f_\phi$.
\end{proof}


\subsubsection{Designing LAN as Neural Network Layer}
We now describe the design of LAN as neural network layer informed by the preceding analysis. The process consists of six steps: (i) sparse input curation; (ii) gating mechanism; (iii) handling \(\Delta t\); (iv) attention weights; (v) attention output; and (vi) extension to multi-head. Figure~\ref{fig:fluid_architecture} illustrate the architectural overview of LAN.

\begin{wrapfigure}{r}{0.5\textwidth}
\vspace{-8mm}
\begin{minipage}{0.48\textwidth}
\begin{algorithm}[H]
\small
\caption{Sparse Top-\emph{K} Pairwise Concatenation}
\label{algo:sparse_topk}
\begin{algorithmic}
\Require Keys $K \in \mathbb{R}^{B \times H \times T_k \times D}$,  
Top-\emph{K} value $K$
\Ensure concatenated tensor  
$U \in \mathbb{R}^{B \times H \times T_q \times K_{\text{eff}} \times 2D}$
\State Scores: $S \gets Q \cdot K^\top$
\State Effective Top-\emph{K}: $K_{\text{eff}} \gets \min(K, T_k)$
\State Indices: $I_{\text{topk}} \gets \text{top\_k}(S, K_{\text{eff}})$
\State Gather: $K_{\text{selected}} \gets \text{gather}(K, I_{\text{topk}})\in 
\mathbb{R}^{B \times H \times T_q \times K_{\text{eff}} \times D}$
\State Tiled: $Q_{\text{tiled}} \gets \text{tile}(Q, K_{\text{eff}})\in 
\mathbb{R}^{B \times H \times T_q \times K_{\text{eff}} \times D}$
\State Concat: $U_{\text{topk}} \gets [\,Q_{\text{tiled}};\; K_{\text{selected}}\,] \in 
\mathbb{R}^{B \times H \times T_q \times K_{\text{eff}} \times 2D}$
\State \Return $U_{\text{topk}}$
\end{algorithmic}
\end{algorithm}
\end{minipage}
\end{wrapfigure}

\textbf{Sparse Input Curation:} We experimented with different strategies for constructing query--key inputs. Initially, we implemented full pairwise concatenation, where queries $Q \in \mathbb{R}^{B \times H \times T_q \times D}$ are combined with all keys $K \in \mathbb{R}^{B \times H \times T_k \times D}$ to form a joint tensor $U \in \mathbb{R}^{B \times H \times T_q \times T_k \times 2D}$. While this preserved complete feature information and enabled expressive, learnable similarity functions, it was memory-intensive, making it impractical for longer sequences. To mitigate this, we applied a sparse Top-\emph{K} optimization: for each query, we compute pairwise scores $S = Q \cdot K^\top \in \mathbb{R}^{B \times H \times T_q \times T_k}$, select the Top-$K_{\text{eff}} = \min(K, T_k)$ keys, and construct concatenated pairs $U_{\text{topk}} \in \mathbb{R}^{B \times H \times T_q \times K_{\text{eff}} \times 2D}$. This approach preserves the most relevant interactions while substantially reducing memory requirements in the concatenation and subsequent processing stages. Algorithm~\ref{algo:sparse_topk} outlines the steps required for input curation.

\textbf{Gating Mechanism:} Starting from Eqn.~\ref{eq:liquid_ode}, consider the trajectory of a query--key pair with initial condition $a_0=0$:  
\begin{equation}
a_{n+1} = a_n + \Delta t(-f_\tau a_n + f_\phi),
\end{equation}
where the gates are parameterized using two independent projection heads over a shared recurrent-based\(\mathcal{R}_g\) (either LSTM or GRU). It allows the gates to capture temporal dependencies and helps achieve faster convergence during training.  
\begin{align}
f_\phi(\mathbf{u_t}) &=  \mathrm{tanh}(\mathcal{R}_g( \mathbf{u_t},t,\theta_\phi))\\
f_\tau(\mathbf{u_t}) &= \mathrm{softplus}(\mathcal{R}_g( \mathbf{u_t},t,\theta_\tau)) + \varepsilon, \quad \varepsilon > 0
\label{eq:tau_calc}
\end{align}
Here, $f_\phi$ serves as a \emph{content-target} gate with \(\mathrm{tanh}\) activation function, allowing both negative and positive logits to proceed, while $f_\tau$ is a strictly positive \emph{time--constant} gate, controlling both the rate of convergence and the steady-state amplitude. Intuitively, this shared recurrent gating: $f_\phi$ decides \emph{what} content to emphasize, while $f_\tau$ governs \emph{how quickly} and \emph{to what extent} it appears. 

\textbf{Handling (\(\Delta t\)):} The attention dynamics can become stiff when $f_\tau$ reaches large values, and naive \textit{explicit Euler} integration may become numerically unstable. We therefore enforce the stability constraint (Lemma~\ref{lemma:euler_stability}) by adaptively clamping the effective \(\Delta t\) based on the maximum predicted $f_\tau$ at each forward pass. This simple control prevents exploding \(a_n\) and numerical divergence. Empirically, it leads to more stable training, faster convergence, and improved performance compared to using a fixed, unconstrained step size.

\textbf{Attention Weights:} Normalizing across all keys via $\mathrm{softmax}$ yields attention weights \(\alpha_t = \mathrm{softmax}(a_t),\) defining a valid probability distribution where $f_\phi$ amplifies or suppresses content alignments, and $f_\tau$ shapes both the speed and saturation of these preferences.

\textbf{Attention output:} Finally, the attention output is computed by multiplying the attention weights with the value matrix:
\begin{equation}
\text{LAN}(q, k, v) = \alpha_t \cdot \mathbf{v_t}
\end{equation}
\textbf{Extension to Multi-head:} To scale this mechanism to multi-head attention, we project the input sequence into $H$ independent subspaces (heads) of dimension $d_\text{model}/H$, yielding query, key, and value tensors $(q^{(h)}, k^{(h)}, v^{(h)})$ for $h \in \{1, \dots, H\}$. For each head, pairwise logits are computed using Eqn.~\ref{eq:liquid_euler}, followed by the $\mathrm{softmax}$ normalization to calculate attention weights. The resulting attention weights $\alpha^{(h)}_t$ are then used to multiply with the value vector $v^{(h)}$, producing head-specific attention outputs. Finally, these outputs are concatenated and linearly projected back into the model dimension. This formulation ensures that each head learns distinct dynamic compatibilities governed by its own parameterization of $f_\phi$ and $f_\tau$, while the aggregation across heads preserves the expressive capacity of the standard multi-head attention mechanism.


\begin{wrapfigure}{r}{0.5\columnwidth}
\centering
\vspace{-5mm}
\includegraphics[width=0.49\columnwidth]{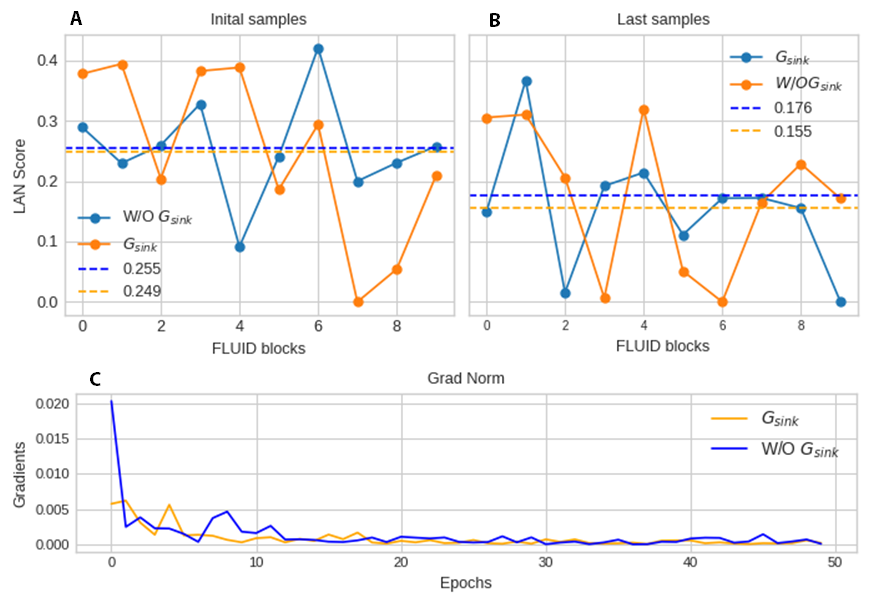}
\caption{Comparison of attention mass distribution with and without \(G_{\text{sink}}\). \textbf{Setup:} The test is conducted on XJTU-SY Bearing~1 from first operating condition with FLUID, using a 10-blocks Transformer trained for 50 epochs with full pairwise query-key concatenation and standard residual connections. \textbf{Test:} \textbf{(A)} The first 200 samples, showing the attention sink, and \textbf{(B)} the last 200 samples, showing the performance gains. Attention without \(G_{\text{sink}}\) in \textbf{(A)} shows a 2.43\% decrease, dropping from 0.255 to 0.249 on average, while in \textbf{(B)} the last samples improve from 0.155 to 0.176, corresponding to a 13.5\% gain. \textbf{(C)} Gradient norm comparison over 50 epochs, showing that including \(G_{\text{sink}}\) reduces initial gradient spikes and stabilizes training.}
\label{fig:sink_result}
\vspace{-10pt}
\end{wrapfigure}

\subsection{Attention-sink gate (\(G_{\text{sink}}\))}
\textbf{Attention sink:} Attention-sink occurs when the Transformer allocates disproportionate attention to early tokens due to $\mathrm{softmax}$ normalization, even when these tokens carry little relevant information~\cite{xiao2023efficient}. In LAN, an attention sink may correspond to a node that tends to accumulate disproportionately large attention, acting as a potential stable attractor of the dynamics. In irregular time-series, this effect is amplified by non-uniform sampling and varying temporal gaps. Early observations, regardless of temporal significance, can accumulate disproportionate attention. This introduces a systematic bias that overemphasizes initial events, even though informative patterns are often sparse, temporally localized, or interval-dependent. Accounting for attention sinks is therefore essential to ensure attention reflects temporal relevance rather than artifacts of sequence ordering or dynamical attractors.

Several approaches have been proposed to mitigate attention sinks, including replacing $\mathrm{softmax}$ with unnormalized sigmoid attention~\cite{gu2024attention}, clipping or calibrating attention scores~\cite{bondarenko2023quantizable}, and modifying the $\mathrm{softmax}$~\cite{zuhri2025softpick} computation. While sparse Top-\emph{K} attention~\cite{zhao2019explicit} alleviates the issue to some extent, it does not eliminate it. A recent study showed that adding a gating mechanism after $\mathrm{softmax}$ fully addresses the problem~\cite{qiu2025gated} and stabilizes transformer training. In LAN, we adopt this technique by adding a sigmoid gate after the multihead output by computing query-dependent sigmoid gating scores, which are then applied via element-wise multiplication ($\odot$) to the multihead output as illustrated in Figure~\ref{fig:fluid_architecture}. \textit{Intuitively}, this introduces input-dependent sparsity that selectively suppresses irrelevant contributions. As shown in Figure~\ref{fig:sink_result}(A), this reduces the attention allocated to initial nodes by 2.43\%, decreasing from 0.255 to 0.249, thereby mitigating the attention sink effect. In contrast, Figure~\ref{fig:sink_result}(B) shows that attention to later nodes increases by 13.5\%, rising from 0.155 to 0.176, indicating a more balanced distribution of attention. Furthermore, Figure~\ref{fig:sink_result}(C) demonstrates that incorporating $G_{\text{sink}}$ reduces early gradient spikes and improves training stability, consistent with the role of sparse gating in controlling massive activations. Formally, the post \(G_{\text{sink}}\) is computed as:
\begin{equation}
O_{\text{LAN}} = \sigma(\mathbf{x_t}) \odot \Big(\mathrm{Concat}(h_1, h_2, \dots,h_H) W_g + b_g\Big) 
\label{eq:gsink}
\end{equation}

\subsection{Hyper-Connections (\(\mathcal{HC}\))}
\textbf{Seesaw effect:} Residual connections are widely used in deep neural networks to facilitate identity mapping and form a core component of the Transformer architectures. However, they exhibit a trade-off known as the seesaw effect, reflecting an inherent tension between gradient propagation and representation collapse due to fixed residual strength. Strong residual paths (PreNorm) promote stable gradient flow but cause adjacent layers to produce highly similar representations, reducing the effective contribution of network depth. Conversely, weaker residual influence (PostNorm) increases representational diversity but raises the risk of vanishing gradients. Consequently, improving one aspect tends to degrade the other, revealing a fundamental limitation of conventional residual connections~\cite{zhu2024hyper}.

To accommodate the seesaw effect in FLUID, we replace residual connections with hyper-connections ($\mathcal{HC}$). Formally,  Let $h^{(l-1)} \in \mathbb{R}^d$ denote the hidden representation entering layer $l$, with $x^{(0)}$ as the network input. $\mathcal{HC}$ expand this representation into $n$ parallel streams. Specifically, we initialize
\begin{equation}
\mathbf{X}^{(0)} =
\begin{bmatrix}
x^{(0)} \\
x^{(0)} \\
\vdots \\
x^{(0)}
\end{bmatrix}
\in \mathbb{R}^{n \times d},
\end{equation}
Here, $n$ is the expansion rate. For layer $l$, the input hyper-hidden matrix is
\begin{equation}
\mathbf{H}^{(k-1)} =
\begin{bmatrix}
h^{(l-1)}_1 \\
h^{(l-1)}_2 \\
\vdots \\
h^{(l-1)}_n
\end{bmatrix}
\in \mathbb{R}^{n \times d},
\end{equation}
For clarity, we omit the layer index $l$ when the context is unambiguous and write $\mathbf{H}$. $\mathcal{HC}$ are defined by a structured matrix
\begin{equation}
\begin{split}
\mathcal{HC} &=
\begin{pmatrix}
0_{1\times1} & \mathbf{B} \\
\mathbf{A_m} & \mathbf{A_r} 
\end{pmatrix}  = 
\begin{pmatrix}
0 & \beta_1 & \beta_2 & \cdots & \beta_n \\
\alpha_{1,0} & \alpha_{1,1} & \alpha_{1,2} & \cdots & \alpha_{1,n} \\
\alpha_{2,0} & \alpha_{2,1} & \alpha_{2,2} & \cdots & \alpha_{2,n} \\
\vdots & \vdots & \vdots & \ddots & \vdots \\
\alpha_{n,0} & \alpha_{n,1} & \alpha_{n,2} & \cdots & \alpha_{n,n} 
\end{pmatrix} \in \mathbb{R}^{(n+1)\times(n+1)}, 
\label{eq:hc_matrix}
\end{split}
\end{equation}
where $\mathbf{B} \in \mathbb{R}^{n}$ controls the contribution of the current layer, $\mathbf{A_m} \in \mathbb{R}^{n}$ aggregates hyper-hidden states into a single input, and $\mathbf{A_r} \in \mathbb{R}^{n \times n}$ defines residual mixing across streams.

\begin{wrapfigure}{r}{0.5\textwidth}
\begin{minipage}{0.47\textwidth}
\vspace{-5mm}
\begin{algorithm}[H]
\small
\caption{Network with Hyper-Connections}
\label{algo:network_hyper_connections}
\begin{algorithmic}
\Require Initial input vector $x^{(0)} \in \mathbb{R}^d$, Expansion rate $n$
\Ensure Final output $y$
\State Initialize: $\mathbf{X}^{(0)} \gets x^{(0)} \ x^{(0)} \ \dots \ x^{(0)})^\top \in \mathbb{R}^{n \times d}$
\For{$l = 1$ to $\mathscr{L}$}
\State $\mathbf{H} \gets \mathbf{H}^{(l-1)}$
\State Width Connections: $( x_0 \ \mathbf{H'}) \gets {\mathcal{WC}}^{l}\top \mathbf{H}$
\State Layer Computation: $x'_0 \gets \mathscr{L}^l(x_0)$
\State Depth Connections: $\hat{\mathbf{H}} \gets {\mathbf{B}^l\top} (x'_0)^\top + \mathbf{H^{'}}$
\State $\mathbf{H}^l \gets \mathbf{\hat{H}}$
\EndFor
\State \textbf{Final Output:}
\State $h^{\mathscr{L}} \gets$ sum rows of $H^{\mathscr{L}}$
\State $h^{\mathscr{L}} \gets$ LayerNorm$(h^{\mathscr{L}})$
\State $y \gets$ Output Layer$(h^{\mathscr{L}})$
\State \Return $y$
\end{algorithmic}
\end{algorithm}
\end{minipage}
\vspace{5mm}
\end{wrapfigure}

\textbf{Layer Computation: } Let $\mathscr{L}(\cdot)$ denote a FLUID transformer sublayer (LAN or FFN). The aggregated input to the sublayer is computed as
\begin{equation}
\mathbf{x_0}^\top = \mathbf{A_m}^\top \mathbf{H}.
\label{eq:hc_input}
\end{equation}
Residual propagation across hyper-hidden states is given by
\begin{equation}
\mathbf{H}_{\text{r}} = \mathbf{A_r}^\top \mathbf{H},
\label{eq:hc_residual}
\end{equation}
The final hyper-hidden representation is obtained by broadcasting the sub-layer output back into the hyper-hidden space and adding the residual connection, giving
\begin{equation}
\widehat{\mathbf{X}} = \mathbf{B}^\top \mathscr{L}(\mathbf{x_0}^\top) + \mathbf{H}_{\text{r}}
\label{eq:hc_combined}
\end{equation}
This formulation is also known as static hyper-connections (\(\mathcal{HC}_{\text{static}}\))~\cite{zhu2024hyper}. \textit{Depth Connections~(\(\mathcal{DC}\))} correspond to residual-like pathways and can be summarized by 
\begin{equation}
\mathcal{DC} =
\begin{pmatrix}
\mathbf{B} \\
\operatorname{diag}(\mathbf{A_r})
\end{pmatrix}
\in \mathbb{R}^{2 \times n},
\end{equation}
where the first row weights the current layer output, and the second row weights the  identity residual. \textit{Width connections~(\(\mathcal{WC}\))} enable information exchange across streams, and are defined as
\begin{equation}
\mathcal{WC} =
\begin{pmatrix}
\mathbf{A_m} & \mathbf{A_r}
\end{pmatrix}
\in \mathbb{R}^{n \times (n+1)}.
\end{equation}

Algorithm~\ref{algo:network_hyper_connections} summarizes key steps to implement a network with $\mathcal{HC}$.

\subsubsection{Liquid Hyper-Connections (\(\mathcal{HC}_{\text{liquid}}\)):} 
To enable input-dependent (Liquid) connectivity, hyper-connections are defined as functions of hyper-input ($\mathbf{X}$) as
\begin{equation}
\mathcal{HC}_{\text{liquid}}(\mathbf{X}) =
\begin{pmatrix}
0_{1\times1} & \mathcal{B}(\mathbf{X}) \\
\mathcal{A}_m(\mathbf{X}) & \mathcal{A}_r(\mathbf{X})
\end{pmatrix}.
\end{equation}
The layer output is computed identically to the static case, replacing fixed parameters with their liquid counterparts. To stabilize the training process, the input is first normalized, followed by a $\mathrm{tanh}$ activation, scaled by a small initial learning rate. The following equation details how these parameters are computed:
\begin{align}
\widetilde{\mathbf{X}} &= \operatorname{Norm}(\mathbf{X}),\\
\mathcal{B}(\mathbf{X}) &= \mathbf{B} + s_b \odot \tanh(\widetilde{\mathbf{X}}\mathbf{W_b}), \\
\mathcal{A}_m(\mathbf{X}) &= \mathbf{A_m} + s_a \odot \tanh(\widetilde{\mathbf{X}}\mathbf{W_m}), \\
\mathcal{A}_r(\mathbf{X}) &= \mathbf{A_r} + s_a \odot \tanh(\widetilde{\mathbf{X}}\mathbf{W_r}),
\end{align}
where $\mathbf{W_b}$, $\mathbf{W_m}$, and $\mathbf{W_r}$ are learnable projection matrices, $s_b$ and $s_a$ are learnable scaling factors initialized to small values~\cite{zhu2024hyper}.


\begin{table}[t]
\centering
\caption{Performance Score comparison of all models.}
\label{tab:all_results}
\vspace{1mm}
\resizebox{\textwidth}{!}{%
\begin{tabular}{lcccccccccc}
\toprule

\multirow{2}{*}{\textbf{Model}} & \multicolumn{2}{c}{\textbf{Irregular Time-series}} & \multicolumn{2}{c}{\textbf{Long-range Modeling}} & \multicolumn{2}{c}{\textbf{AVs Control}} & \multicolumn{3}{c}{\textbf{Physical Modeling}}\\

\cmidrule(lr){2-3} \cmidrule(lr){4-5} \cmidrule(lr){6-7} \cmidrule(lr){8-10}
 & \textbf{Spiral (↓)} & \textbf{E-MNIST (↑)} & \textbf{ETTm1 (↓)} & \textbf{Jena-Climate (↓)}  & \textbf{Udacity (↓)} & \textbf{CarRacing (↑)} & \textbf{XJTU-SY (↓)} & \textbf{HUST (↓)} & \textbf{PRONOSTIA (↓)}  \\
\midrule
CT-RNN
& 0.0063\textsuperscript{\scriptsize $\pm$0.0008}
& 95.18\textsuperscript{\scriptsize $\pm$0.20}
& 0.0158\textsuperscript{\scriptsize $\pm$0.0125}
& 0.0591\textsuperscript{\scriptsize $\pm$0.0146}
& 0.0202\textsuperscript{\scriptsize $\pm$0.0017}
& \textbf{80.80\textsuperscript{\scriptsize $\pm$0.27}}
& 45.48\textsuperscript{\scriptsize $\pm$20.82}
& 47.35\textsuperscript{\scriptsize $\pm$23.42}
& 222.44\textsuperscript{\scriptsize $\pm$29.88} \\

GRU-ODE
& 0.0048\textsuperscript{\scriptsize $\pm$0.0001}
& 96.04\textsuperscript{\scriptsize $\pm$0.13}
& 0.0783\textsuperscript{\scriptsize $\pm$0.0000}
& 0.0666\textsuperscript{\scriptsize $\pm$0.0074}
& 0.0197\textsuperscript{\scriptsize $\pm$0.0027}
& 80.29\textsuperscript{\scriptsize $\pm$0.72}
& 25.26\textsuperscript{\scriptsize $\pm$9.65}
& 51.04\textsuperscript{\scriptsize $\pm$8.48}
& 195.01\textsuperscript{\scriptsize $\pm$58.99} \\

PhasedLSTM
& 0.1038\textsuperscript{\scriptsize $\pm$0.0107}
& 95.79\textsuperscript{\scriptsize $\pm$0.14}
& 0.1963\textsuperscript{\scriptsize $\pm$0.0088}
& 0.0787\textsuperscript{\scriptsize $\pm$0.0058}
& 0.0184\textsuperscript{\scriptsize $\pm$0.0011}
& 80.65\textsuperscript{\scriptsize $\pm$0.38}
& \textbf{23.19\textsuperscript{\scriptsize $\pm$8.08}}
& 57.96\textsuperscript{\scriptsize $\pm$28.65}
& 233.09\textsuperscript{\scriptsize $\pm$28.54} \\

mmRNN
& 0.0078\textsuperscript{\scriptsize $\pm$0.0009}
& 95.74\textsuperscript{\scriptsize $\pm$0.27}
& 0.0749\textsuperscript{\scriptsize $\pm$0.0313}
& 0.2748\textsuperscript{\scriptsize $\pm$0.0354}
& 0.0213\textsuperscript{\scriptsize $\pm$0.0039}
& 80.13\textsuperscript{\scriptsize $\pm$0.54}
& 34.21\textsuperscript{\scriptsize $\pm$12.61}
& 46.69\textsuperscript{\scriptsize $\pm$19.71}
& 165.59\textsuperscript{\scriptsize $\pm$29.65} \\

LTC-FC
& 0.0050\textsuperscript{\scriptsize $\pm$0.0001}
& 81.25\textsuperscript{\scriptsize $\pm$0.00}
& 0.0659\textsuperscript{\scriptsize $\pm$0.0201}
& 0.1536\textsuperscript{\scriptsize $\pm$0.0000}
& 0.0252\textsuperscript{\scriptsize $\pm$0.0018}
& 76.37\textsuperscript{\scriptsize $\pm$3.01}
& 28.34\textsuperscript{\scriptsize $\pm$13.71}
& 40.31\textsuperscript{\scriptsize $\pm$8.20}
& 144.25\textsuperscript{\scriptsize $\pm$21.96} \\

LTC-NCP
& 0.1312\textsuperscript{\scriptsize $\pm$0.1537}
& 77.77\textsuperscript{\scriptsize $\pm$3.26}
& 0.6246\textsuperscript{\scriptsize $\pm$0.3115}
& 0.2437\textsuperscript{\scriptsize $\pm$0.1005}
& 0.0217\textsuperscript{\scriptsize $\pm$0.0010}
& 79.19\textsuperscript{\scriptsize $\pm$1.55}
& 127.19\textsuperscript{\scriptsize $\pm$5.16}
& 43.08\textsuperscript{\scriptsize $\pm$15.86}
& 138.05\textsuperscript{\scriptsize $\pm$5.09} \\

CfC-FC
& 0.0101\textsuperscript{\scriptsize $\pm$0.0010}
& 94.16\textsuperscript{\scriptsize $\pm$0.49}
& 0.0138\textsuperscript{\scriptsize $\pm$0.0000}
& 0.0672\textsuperscript{\scriptsize $\pm$0.0000}
& 0.0202\textsuperscript{\scriptsize $\pm$0.0020}
& 80.59\textsuperscript{\scriptsize $\pm$0.33}
& 34.34\textsuperscript{\scriptsize $\pm$13.63}
& 56.60\textsuperscript{\scriptsize $\pm$16.35}
& 162.24\textsuperscript{\scriptsize $\pm$32.80} \\

CfC-NCP
& 0.1307\textsuperscript{\scriptsize $\pm$0.1541}
& 80.60\textsuperscript{\scriptsize $\pm$1.47}
& 1.3941\textsuperscript{\scriptsize $\pm$0.3859}
& 138.41\textsuperscript{\scriptsize $\pm$115.21}
& 0.0209\textsuperscript{\scriptsize $\pm$0.0014}
& 27.75\textsuperscript{\scriptsize $\pm$12.29}
& 114.32\textsuperscript{\scriptsize $\pm$26.80}
& 53.25\textsuperscript{\scriptsize $\pm$19.43}
& 131.90\textsuperscript{\scriptsize $\pm$33.51} \\

DeepState
& 0.3172\textsuperscript{\scriptsize $\pm$0.0028}
& 96.10\textsuperscript{\scriptsize $\pm$0.07}
& 0.0032\textsuperscript{\scriptsize $\pm$0.0001}
& 0.5688\textsuperscript{\scriptsize $\pm$0.1766}
& 0.0181\textsuperscript{\scriptsize $\pm$0.0006}
& 80.25\textsuperscript{\scriptsize $\pm$0.20}
& \textbf{23.60\textsuperscript{\scriptsize $\pm$9.30}}
& 53.48\textsuperscript{\scriptsize $\pm$11.29}
& 144.24\textsuperscript{\scriptsize $\pm$9.66} \\

S4
& 0.3194\textsuperscript{\scriptsize $\pm$0.0001}
& 79.25\textsuperscript{\scriptsize $\pm$0.00}
& 0.2194\textsuperscript{\scriptsize $\pm$0.0180}
& 1.9298\textsuperscript{\scriptsize $\pm$1.8842}
& 0.0192\textsuperscript{\scriptsize $\pm$0.0013}
& 32.16\textsuperscript{\scriptsize $\pm$17.46}
& 124.85\textsuperscript{\scriptsize $\pm$16.39}
& 45.43\textsuperscript{\scriptsize $\pm$15.31}
& 128.35\textsuperscript{\scriptsize $\pm$18.36} \\

\noalign{\vskip 2.5pt}
\hdashline
\noalign{\vskip 2.5pt}

SDPA-Transformer
& 0.3194\textsuperscript{\scriptsize $\pm$0.0000}
& 96.08\textsuperscript{\scriptsize $\pm$0.31}
& 0.0043\textsuperscript{\scriptsize $\pm$0.0010}
& 0.0052\textsuperscript{\scriptsize $\pm$0.0004}
& 0.0205\textsuperscript{\scriptsize $\pm$0.0012}
& 80.80\textsuperscript{\scriptsize $\pm$0.23}
& 121.32\textsuperscript{\scriptsize $\pm$10.59}
& 103.53\textsuperscript{\scriptsize $\pm$19.23}
& 181.31\textsuperscript{\scriptsize $\pm$10.59} \\

LinFormer
& 0.2220\textsuperscript{\scriptsize $\pm$0.0811}
& 95.87\textsuperscript{\scriptsize $\pm$0.14}
& 0.0087\textsuperscript{\scriptsize $\pm$0.0008}
& 0.0087\textsuperscript{\scriptsize $\pm$0.0002}
& 0.0192\textsuperscript{\scriptsize $\pm$0.0017}
& 80.16\textsuperscript{\scriptsize $\pm$0.30}
& 32.92\textsuperscript{\scriptsize $\pm$15.27}
& 38.56\textsuperscript{\scriptsize $\pm$5.13}
& 296.39\textsuperscript{\scriptsize $\pm$55.74} \\

PerFormer
& 0.0311\textsuperscript{\scriptsize $\pm$0.0029}
& 96.22\textsuperscript{\scriptsize $\pm$0.18}
& 0.0027\textsuperscript{\scriptsize $\pm$0.0001}
& 2.6132\textsuperscript{\scriptsize $\pm$1.4312}
& 0.0179\textsuperscript{\scriptsize $\pm$0.0011}
& 62.86\textsuperscript{\scriptsize $\pm$9.06}
& 37.52\textsuperscript{\scriptsize $\pm$9.61}
& 41.52\textsuperscript{\scriptsize $\pm$5.24}
& 154.08\textsuperscript{\scriptsize $\pm$19.02} \\

\noalign{\vskip 2.5pt}
\hdashline
\noalign{\vskip 2.5pt}

mTAN
& 0.3194\textsuperscript{\scriptsize $\pm$0.0001}
& \textbf{96.30\textsuperscript{\scriptsize $\pm$0.25}}
& 0.8368\textsuperscript{\scriptsize $\pm$0.0808}
& 0.0181\textsuperscript{\scriptsize $\pm$0.0000}
& 0.0185\textsuperscript{\scriptsize $\pm$0.0007}
& 80.79\textsuperscript{\scriptsize $\pm$0.24}
& 33.71\textsuperscript{\scriptsize $\pm$10.49}
& 48.85\textsuperscript{\scriptsize $\pm$10.47}
& 141.84\textsuperscript{\scriptsize $\pm$47.55} \\

ODEFormer
& 0.0358\textsuperscript{\scriptsize $\pm$0.0145}
& 96.22\textsuperscript{\scriptsize $\pm$0.20}
& 0.0029\textsuperscript{\scriptsize $\pm$0.0002}
& 0.1210\textsuperscript{\scriptsize $\pm$0.0795}
& 0.0188\textsuperscript{\scriptsize $\pm$0.0015}
& 80.54\textsuperscript{\scriptsize $\pm$0.40}
& 32.35\textsuperscript{\scriptsize $\pm$6.07}
& 40.28\textsuperscript{\scriptsize $\pm$9.98}
& 158.34\textsuperscript{\scriptsize $\pm$24.87} \\

CTA
& 0.2723\textsuperscript{\scriptsize $\pm$0.0048}
& 95.97\textsuperscript{\scriptsize $\pm$0.14}
& 0.0028\textsuperscript{\scriptsize $\pm$0.0003}
& 0.1051\textsuperscript{\scriptsize $\pm$0.0013}
& 0.0201\textsuperscript{\scriptsize $\pm$0.0025}
& 80.43\textsuperscript{\scriptsize $\pm$0.37}
& 32.21\textsuperscript{\scriptsize $\pm$11.06}
& 47.64\textsuperscript{\scriptsize $\pm$14.77}
& 118.61\textsuperscript{\scriptsize $\pm$22.58} \\

PDE-Attention
& 0.3138\textsuperscript{\scriptsize $\pm$0.0001}
& 95.95\textsuperscript{\scriptsize $\pm$0.22}
& 0.0028\textsuperscript{\scriptsize $\pm$0.0002}
& 0.1854\textsuperscript{\scriptsize $\pm$0.0257}
& 0.0197\textsuperscript{\scriptsize $\pm$0.0019}
& 75.95\textsuperscript{\scriptsize $\pm$8.80}
& 136.64\textsuperscript{\scriptsize $\pm$28.63}
& 44.03\textsuperscript{\scriptsize $\pm$11.26}
& 138.95\textsuperscript{\scriptsize $\pm$29.21} \\

ContiFormer
& 0.0053\textsuperscript{\scriptsize $\pm$0.0003}
& 96.15\textsuperscript{\scriptsize $\pm$0.23}
& 0.0026\textsuperscript{\scriptsize $\pm$0.0002}
& 0.0040\textsuperscript{\scriptsize $\pm$0.0002}
& 0.0188\textsuperscript{\scriptsize $\pm$0.0010}
& 50.47\textsuperscript{\scriptsize $\pm$0.50}
& 58.24\textsuperscript{\scriptsize $\pm$8.14}
& 88.57\textsuperscript{\scriptsize $\pm$10.16}
& 182.96\textsuperscript{\scriptsize $\pm$8.91} \\

OT-Transformer
& 0.0070\textsuperscript{\scriptsize $\pm$0.0006}
& 96.23\textsuperscript{\scriptsize $\pm$0.20}
& 0.0032\textsuperscript{\scriptsize $\pm$0.0003}
& 0.9201\textsuperscript{\scriptsize $\pm$0.8453}
& 0.0201\textsuperscript{\scriptsize $\pm$0.0010}
& 80.44\textsuperscript{\scriptsize $\pm$0.37}
& 153.35\textsuperscript{\scriptsize $\pm$13.85}
& \textbf{35.81\textsuperscript{\scriptsize $\pm$7.53}}
& 155.72\textsuperscript{\scriptsize $\pm$14.66} \\

\midrule

FLUID {\color{gray!70}$\mid$RES$\mid$}
& 0.0048\textsuperscript{\scriptsize $\pm$0.0001}
& 96.80\textsuperscript{\scriptsize $\pm$1.46}
& 0.0024\textsuperscript{\scriptsize $\pm$0.0001}
& 0.0023\textsuperscript{\scriptsize $\pm$0.0003}
& 0.0182\textsuperscript{\scriptsize $\pm$0.0010}
& 80.79\textsuperscript{\scriptsize $\pm$0.12}
& 26.92\textsuperscript{\scriptsize $\pm$10.88}
& 49.61\textsuperscript{\scriptsize $\pm$13.41}
& 114.55\textsuperscript{\scriptsize $\pm$16.29} \\

FLUID {\color{gray!70}$\mid$W/O~\(G_{\text{sink}}\)$\mid$}
& 0.0048\textsuperscript{\scriptsize $\pm$0.0002}
& 96.03\textsuperscript{\scriptsize $\pm$0.16}
& \textbf{0.0023\textsuperscript{\scriptsize $\pm$0.0001}}
& 0.0022\textsuperscript{\scriptsize $\pm$0.0002}
& 0.0187\textsuperscript{\scriptsize $\pm$0.0014}
& \textbf{80.83\textsuperscript{\scriptsize $\pm$0.26}}
& 24.19\textsuperscript{\scriptsize $\pm$11.40}
& 34.71\textsuperscript{\scriptsize $\pm$5.07}
& \textbf{112.51\textsuperscript{\scriptsize $\pm$15.53}} \\

FLUID {\color{gray!70}$\mid$\(\mathcal{HC}_{\text{static}}\)$\times$2$\mid$}
& 0.0047\textsuperscript{\scriptsize $\pm$0.0001}
& 96.12\textsuperscript{\scriptsize $\pm$0.25}
& 0.0028\textsuperscript{\scriptsize $\pm$0.0005}
& 0.0022\textsuperscript{\scriptsize $\pm$0.0005}
& \textbf{0.0177\textsuperscript{\scriptsize $\pm$0.0014}}
& \textbf{80.83\textsuperscript{\scriptsize $\pm$0.11}}
& 26.48\textsuperscript{\scriptsize $\pm$10.47}
& 47.41\textsuperscript{\scriptsize $\pm$14.79}
& 118.80\textsuperscript{\scriptsize $\pm$20.45} \\

FLUID {\color{gray!70}$\mid$\(\mathcal{HC}_{\text{static}}\)$\times$4$\mid$}
& 0.0047\textsuperscript{\scriptsize $\pm$0.0002}
& 96.16\textsuperscript{\scriptsize $\pm$0.19}
& 0.0030\textsuperscript{\scriptsize $\pm$0.0008}
& 0.0020\textsuperscript{\scriptsize $\pm$0.0002}
& 0.0183\textsuperscript{\scriptsize $\pm$0.0016}
& 80.73\textsuperscript{\scriptsize $\pm$0.32}
& 28.59\textsuperscript{\scriptsize $\pm$3.61}
& 40.04\textsuperscript{\scriptsize $\pm$8.71}
& \textbf{109.64\textsuperscript{\scriptsize $\pm$10.86}} \\

FLUID {\color{gray!70}$\mid$\(\mathcal{HC}_{\text{static}}\)$\times$8$\mid$}
& 0.0047\textsuperscript{\scriptsize $\pm$0.0002}
& \textbf{96.67\textsuperscript{\scriptsize $\pm$1.66}}
& 0.0025\textsuperscript{\scriptsize $\pm$0.0002}
& 0.0025\textsuperscript{\scriptsize $\pm$0.0002}
& 0.0189\textsuperscript{\scriptsize $\pm$0.0007}
& 80.23\textsuperscript{\scriptsize $\pm$0.42}
& 26.75\textsuperscript{\scriptsize $\pm$11.42}
& 38.51\textsuperscript{\scriptsize $\pm$9.92}
& 112.45\textsuperscript{\scriptsize $\pm$20.01} \\

FLUID {\color{gray!70}$\mid$\(\mathcal{HC}_{\text{liquid}}\)$\times$2$\mid$}
& 0.0047\textsuperscript{\scriptsize $\pm$0.0001}
& 96.08\textsuperscript{\scriptsize $\pm$0.27}
& 0.0024\textsuperscript{\scriptsize $\pm$0.0002}
& 0.0022\textsuperscript{\scriptsize $\pm$0.0003}
& \textbf{0.0175\textsuperscript{\scriptsize $\pm$0.0004}}
& 80.64\textsuperscript{\scriptsize $\pm$0.20}
& 31.00\textsuperscript{\scriptsize $\pm$10.05}
& 40.40\textsuperscript{\scriptsize $\pm$12.20}
& 126.08\textsuperscript{\scriptsize $\pm$20.02} \\

FLUID {\color{gray!70}$\mid$\(\mathcal{HC}_{\text{liquid}}\)$\times$8$\mid$}
& 0.0048\textsuperscript{\scriptsize $\pm$0.0002}
& 96.03\textsuperscript{\scriptsize $\pm$0.16}
& \textbf{0.0023\textsuperscript{\scriptsize $\pm$0.0001}}
& 0.0022\textsuperscript{\scriptsize $\pm$0.0002}
& 0.0187\textsuperscript{\scriptsize $\pm$0.0014}
& \textbf{80.83\textsuperscript{\scriptsize $\pm$0.26}}
& 24.19\textsuperscript{\scriptsize $\pm$11.40}
& \textbf{34.71\textsuperscript{\scriptsize $\pm$5.07}}
& 112.51\textsuperscript{\scriptsize $\pm$15.53} \\

FLUID {\color{gray!70}$\mid$Top-\emph{K}$\times$2$\mid$}

& 0.0047\textsuperscript{\scriptsize $\pm$0.0001}
& 96.13\textsuperscript{\scriptsize $\pm$0.08}
& 0.0024\textsuperscript{\scriptsize $\pm$0.0001}
& \textbf{0.0019\textsuperscript{\scriptsize $\pm$0.0003}}
& 0.0185\textsuperscript{\scriptsize $\pm$0.0007}
& 80.79\textsuperscript{\scriptsize $\pm$0.47}
& 26.62\textsuperscript{\scriptsize $\pm$10.42}
& 47.19\textsuperscript{\scriptsize $\pm$5.50}
& 135.50\textsuperscript{\scriptsize $\pm$26.78} \\

FLUID {\color{gray!70}$\mid$Top-\emph{K}$\times$4$\mid$}
& 0.0049\textsuperscript{\scriptsize $\pm$0.0004}
& 96.15\textsuperscript{\scriptsize $\pm$0.14}
& 0.0024\textsuperscript{\scriptsize $\pm$0.0002}
& 0.0020\textsuperscript{\scriptsize $\pm$0.0002}
& 0.0200\textsuperscript{\scriptsize $\pm$0.0025}
& 80.87\textsuperscript{\scriptsize $\pm$0.19}
& 30.94\textsuperscript{\scriptsize $\pm$7.37}
& 44.96\textsuperscript{\scriptsize $\pm$8.47}
& \textbf{109.39\textsuperscript{\scriptsize $\pm$7.18}} \\

FLUID {\color{gray!70}$\mid$PW$\mid$}
& 0.0047\textsuperscript{\scriptsize $\pm$0.0001}
& 96.72\textsuperscript{\scriptsize $\pm$1.19}
& 0.0023\textsuperscript{\scriptsize $\pm$0.0001}
& 0.0023\textsuperscript{\scriptsize $\pm$0.0004}
& 0.0178\textsuperscript{\scriptsize $\pm$0.0006}
& 80.76\textsuperscript{\scriptsize $\pm$0.34}
& \textbf{23.76\textsuperscript{\scriptsize $\pm$7.49}}
& 41.97\textsuperscript{\scriptsize $\pm$18.44}
& 121.07\textsuperscript{\scriptsize $\pm$21.39} \\

\midrule

FLUID {\color{red!70}$\mid$Main$\mid$}
& \textbf{0.0046\textsuperscript{\scriptsize $\pm$0.0002}}
& \textbf{97.37\textsuperscript{\scriptsize $\pm$1.03}}
& \textbf{0.0022\textsuperscript{\scriptsize $\pm$0.0004}}
& \textbf{0.0021\textsuperscript{\scriptsize $\pm$0.0001}}
& \textbf{0.0174\textsuperscript{\scriptsize $\pm$0.0018}}
& \textbf{80.91\textsuperscript{\scriptsize $\pm$0.23}}
& \textbf{21.29\textsuperscript{\scriptsize $\pm$8.94}}
& \textbf{35.67\textsuperscript{\scriptsize $\pm$9.05}}
& \textbf{111.54\textsuperscript{\scriptsize $\pm$7.43}} \\
\bottomrule
\end{tabular}}
\label{tab:pronostia_metrics}
\begin{minipage}{\textwidth}
\footnotesize
\vspace{2mm}
\textbf{Note:} (↑) higher is better; (↓) lower is better. FLUID {\color{red!70}$\mid$Main$\mid$} uses \(\mathcal{HC}_{\text{liquid}} \times\)4, with Top-\emph{K}$\times$8, and \(G_{\text{sink}}\) enabled.
\end{minipage}
\end{table}

\section{Evaluation}\label{sec:eval}
In this section, we evaluate FLUID in a variety of different learning tasks: (i) irregular time-series modeling; (ii) long-range modeling; (iii) lane-keeping control of autonomous-vehicle; and (iv) learning physical dynamics under scarce data. 

\textbf{Baselines:} We compared FLUID against a range of state-of-the-art CT baselines, including CT-RNNs $\in$ [CT-RNN~\cite{rubanova2019latent}, PhasedLSTM~\cite{neil2016phased}, GRU-ODE~\cite{de2019gru}, mmRNN~\cite{lechner2022mixed}, LTC~\cite{hasani2021liquid}, CfC~\cite{hasani2022closed}, DeepState~\cite{rangapuram2018deep}, S4~\cite{gu2021efficiently}], DT-Transformers $\in$ [SDPA Transformer~\cite{vaswani2023attentionneed}, Linear Attention~\cite{tay2020sparse},Performer~\cite{choromanski2020rethinking}], and CT-Transformers $\in$ [mTAN~\cite{shukla2021multi}, CTA~\cite{chien2021continuous}, ODEFormer~\cite{d2023odeformer}, ContiFormer~\cite{chen2023contiformer}, OT-Transformer\cite{kan2025ot}, PDE-Attention~\cite{zhang2025continuous}].

\textbf{Ablation Configurations: }We also include multiple FLUID ablation configurations.  FLUID {\color{gray!70}$\mid$RES$\mid$} uses standard residual connections. FLUID {\color{gray!70}$\mid$W/O~\(G_{\text{sink}}\)$\mid$} disables the attention-sink gate. FLUID {\color{gray!70}$\mid$\(\mathcal{HC}_{\text{static}}\)$\times$n$\mid$} uses static hyper-connections, while FLUID {\color{gray!70}$\mid$\(\mathcal{HC}_{\text{liquid}}\)$\times$n$\mid$} uses Liquid hyper-connections. FLUID {\color{gray!70}$\mid$Top-\emph{K}$\times$n$\mid$} uses Top-\emph{K} sparsity while FLUID {\color{gray!70}$\mid$PW$\mid$} uses full query-key pairwise concatenation. FLUID {\color{red!70}$\mid$Main$\mid$} uses \(\mathcal{HC}_{\text{liquid}} \times\)4, with Top-\emph{K}$\times$8, and \(G_{\text{sink}}\) enabled.

\textbf{Training and Testing Setup:} All experiments are conducted using 5-fold cross-validation. All competing architectures are trained using BPTT~\cite{lecun1988theoretical} on each fold and evaluated across all folds. We report the mean ($\mu$) and standard ($\sigma$) deviation to quantify predictive uncertainty. To reduce human error and ensure optimal hyperparameters, we use Bayesian optimization~\cite{snoek2012practical} to tune all models for 150 search trials. The optimized hyperparameters are provided in Table~\ref{tab:hparams}.

\begin{table*}[t]
\centering
\caption{Summary of Key Hyperparameters of All Experiments}
\label{tab:hparams}
\resizebox{0.9\textwidth}{!}{%
\begin{tabular}{lcccccc}
\toprule
\textbf{Param.} & \textbf{Spiral} & \textbf{E-MNIST} & \textbf{LRM} & \textbf{CarRacing} & \textbf{Udacity} & \textbf{Physics Learning} \\
\midrule
Conv layers & \textendash & 2×\textbf{1D}(64@5) & -- & 3×TD-\textbf{2D}(10--30@3--5) & 5×\textbf{2D}(24--64@5--3, ELU) & 2×\textbf{1D}(32@3, 16@2) \\
FLUID/baseline & 64-d, 16h & 64-d, 4h & 64-d, 4h  & 64-d, 16h & 100-d, 20h & 16-d, 8h \\
euler\_steps & 10 & 5 & 5 & 5 & 5 & 5 \\
Seq\_len & 1000 & 256 & 96 & 1  & 1 & 1000  \\

Dense & 2(Lin) & 32--10(SM) & -- & 64--5(SM) & 64--1(Lin) & 128--1(Lin) \\
Dropout & \textendash & \textendash & \textendash & 0.2 & 0.5 & \textendash \\
Opt. & AdamW & AdamW & AdamW & Adam & AdamW & AdamW \\
LR & 0.0001 & 0.001 & 0.0001 & 0.001 & 0.001  & 0.0001\\
Loss & MSE & SCE & MSE & SCE & MSE & MSE \\
Metric & MAE & Acc & MSE & Acc & MSE & Score~\cite{nectoux2012pronostia} \\
Batch & 64 & 32 & 64 & 32 & 40 & 32 \\
Epochs & 500 & 30 & 100 & 100 & 15 & 150 \\
\bottomrule
\end{tabular}}
\begin{minipage}{\textwidth}
\footnotesize
\textbf{Note:} SCE = Sparse Categorical Crossentropy; Acc = Accuracy; MAE = Mean Absolute Error; MSE = Mean Squared Error; SM = softmax; Lin = Linear; TD = TimeDistributed; Conv1D/2D = Conv1D/2D; $d$ = model dimension; $h$ = attention heads.
\end{minipage}
\end{table*}

\begin{figure}[ht!]
\centering
\includegraphics[width=1.0\textwidth]{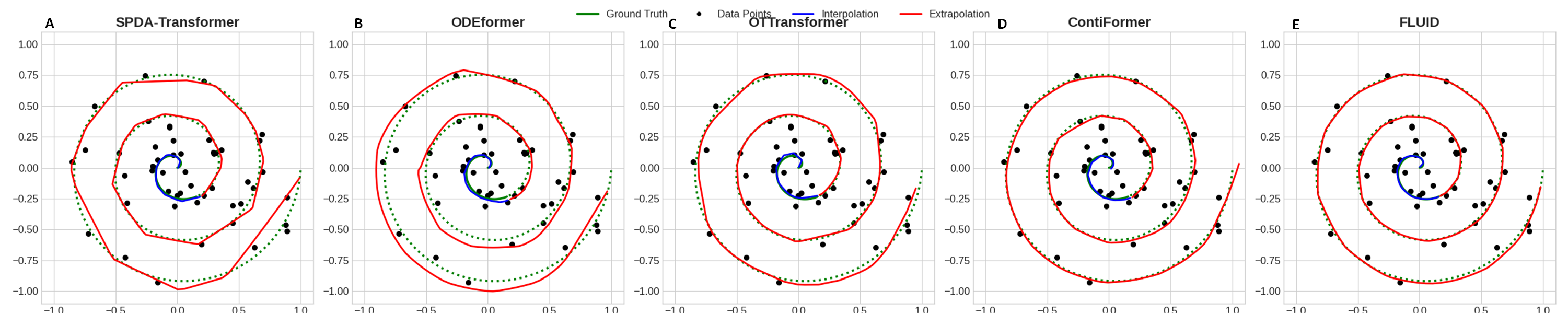}
\caption{Intuitive reconstruction visualization of irregular spiral trajectories: \textbf{(A)} SDPA-Transformer; \textbf{(B)} ODEFormer; \textbf{(C)} OT-Transformer; \textbf{(D)} ContiFormer; \textbf{(E)} FLUID.}
\label{fig:spiral_plots}
\end{figure}

\subsection{Irregular Time-series}
In the first experiment, we evaluate FLUID’s ability to model irregular time-series. We consider two benchmark datasets: (i) Irregular Spiral; and (ii) Event-based MNIST.

\subsubsection{Irregular Spiral} 
The spiral dataset is a commonly used benchmark for intuitively visualizing continuous-time function approximation. We construct the dataset following the procedure in \cite{chen2023contiformer}, generating 300 two-dimensional noisy spirals, each uniformly sampled at 150 time points. To introduce irregularity, we randomly subsample 50 points from each trajectory without replacement. We visualize the reconstruction performance in both interpolation and extrapolation regimes across several Transformer-based models, namely SDPA Transformer, ODEFormer, OT-Transformer, ContiFormer, and FLUID shown in Figure~\ref{fig:spiral_plots}. We observe that most CT-Transformers still rely on discrete attention mechanisms, causing reconstructed trajectories to appear discretized, similar to the SDPA Transformer. In contrast, ContiFormer and FLUID produce smoother spiral trajectories and achieve the highest accuracy in both interpolation and extrapolation settings. 

Table~\ref{tab:all_results} reports the reconstruction MAEs across all the evaluated models. FLUID {\color{red!70}$\mid$Main$\mid$} achieves the lowest error of 0.0046 $\pm$ 0.0002, outperforming the best performing CT-Transformer baseline ContiFormer (0.0053$\pm$0.0003) by 13.2\%. DT-Transformer variants perform substantially worse on this task: the SDPA Transformer records an MAE of 0.3194, whereas LinFormer achieves 0.0222. Among the CT-RNN baselines, the GRU-ODE (0.0048) and LTC-FC (0.0050) perform competitively but remain below the FLUID variants. This finding is consistent with the qualitative analysis showing that FLUID recovers globally coherent spiral geometry in both the interpolation and extrapolation regimes, whereas most baselines produce fragmented or misaligned reconstructions. Within the ablation study, the MAE remains stable across variants (0.0046–0.0048), suggesting that the LAN mechanism is the primary contributor to performance on this benchmark, with \(\mathcal{HC}_{\text{liquid}}\) and $G_{\text{sink}}$ providing incremental improvements.

\subsubsection{Event-based MNIST} 
Introduced by \cite{deng2012mnist}, MNIST is a widely used benchmark dataset in computer vision for handwritten digit recognition. It consists of 70,000 grayscale images of digits from 0 to 9, each with a resolution of 28$\times$28 pixels. The dataset is split into 60,000 training samples and 10,000 testing samples. Following the procedure proposed in \cite{lechner2022mixed}, we convert the grayscale images into irregular time-series representations. This transformation is performed in several steps. First, a threshold is applied to binarize the 8-bit pixel intensities, using 128 as the cutoff between the minimum (0) and maximum (255) intensity values. Next, each 28$\times$28 image is reshaped into a 1-D time-series of length 784. The resulting binary sequence is then encoded into an event-based representation by removing consecutive occurrences of the same value (e.g., $1,1,1,1 \rightarrow 1, t=4$). This encoding introduces a temporal structure and compresses the sequence length from 784 to an average of 53 time steps. Finally, to enable efficient batching and training, each sequence is padded to a fixed length of 256. The time dimension is normalized such that each event corresponds to one unit of time. This process yields a dataset formulated as a per-sequence classification task on irregularly sampled time-series.

Table~\ref{tab:all_results} reports the classification accuracy. FLUID {\color{red!70}$\mid$Main$\mid$} achieves the highest accuracy of 97.37$\pm$1.03, surpassing the second-best baselines mTAN (96.30$\pm$0.25), OT-Transformer (96.23$\pm$0.02), and Performer (96.22$\pm$0.18) by approximately 1.1\%. While this margin is modest in absolute terms, it is meaningful given the saturation of strong performing models in the 96.0–96.3 range. CT-RNN variants perform substantially worse: LTC-NCP (77.77) and CfC-NCP (80.60) exhibit low accuracy and high variance, indicating that wired connectivity structures do not generalize well to compressed event sequences with highly variable interevent intervals. Among ablation configurations, FLUID {\color{red!70}$\mid$Main$\mid$} consistently outperforms FLUID {\color{gray!70}$\mid$RES$\mid$} (96.80) and FLUID {\color{gray!70}$\mid$W/O~\(G_{\text{sink}}\)$\mid$} (96.03). The full model benefits from the combined effect of $G_{\text{sink}}$, which suppresses uninformative nodes, and \(\mathcal{HC}_{\text{liquid}}\), which regulates interlayer signal flow across variable-length event sequences.

\subsection{Long-range Modeling (LRM)}
In the second experiment, we evaluate FLUID’s ability to capture long-range dependencies for accurate forecasting. We consider two real-world benchmark datasets: (i) ETTm1~\cite{informer}; and (ii) Jena Climate~\cite{jena_climate_dataset}.

\subsubsection{ETTm1} 

Ettm1 is a multivariate time- series dataset collected from an electricity transformer in China. It comprises a total of 7 features, including the power load, oil temperature, and other measurements over time. It contains long-range sequences sampled at 15-minute intervals, making it suitable for long-horizon forecasting. We first split the dataset into training (80\%), testing (10\%), and validation (10\%) folds. We condition on the past 12 hours (48 intervals) to predict the subsequent 6 hours (24 intervals). Prior to training, the data are transformed using \texttt{MinmaxScaler}. Post-training, the predictions are then inverse-transformed.

Table~\ref{tab:all_results} reports MSE to assess forecasting accuracy. FLUID {\color{red!70}$\mid$Main$\mid$} achieves the lowest MSE of 0.0022$\pm$0.0004, representing a 15.4\% reduction over the best-performing CT-Transformer baseline ContiFormer (0.0026$\pm$0.0002). The CT-RNN models perform significantly worse, with the GRU-ODE reaching 0.0078 and the CT-RNN reaching 0.0158, indicating that recurrence-based architectures do not scale effectively to long multivariate sequences. Notably, CfC-NCP (1.3941) diverges, highlighting the numerical fragility of closed-form networks under distributional shifts over extended forecasting horizons. As shown in Figure~\ref{fig:forecast}(A), FLUID is the only model whose forecast trajectory faithfully tracks the structure of the Horizon window, whereas all Transformer-based models produce noticeably misaligned projections. Within the ablations, FLUID {\color{gray!70}$\mid$\(\mathcal{HC}_{\text{liquid}}\)$\times$8$\mid$}(0.0023) and FLUID {\color{gray!70}$\mid$RES$\mid$} both improve upon all the baselines, confirming that the LAN attention mechanism drives the core performance gains, while \(\mathcal{HC}_{\text{liquid}}\) provides an additional consistent benefit over both \(\mathcal{HC}_{\text{static}}\) and the standard residual configurations.

\begin{figure}[ht!]
\centering
\includegraphics[width=0.85\textwidth]{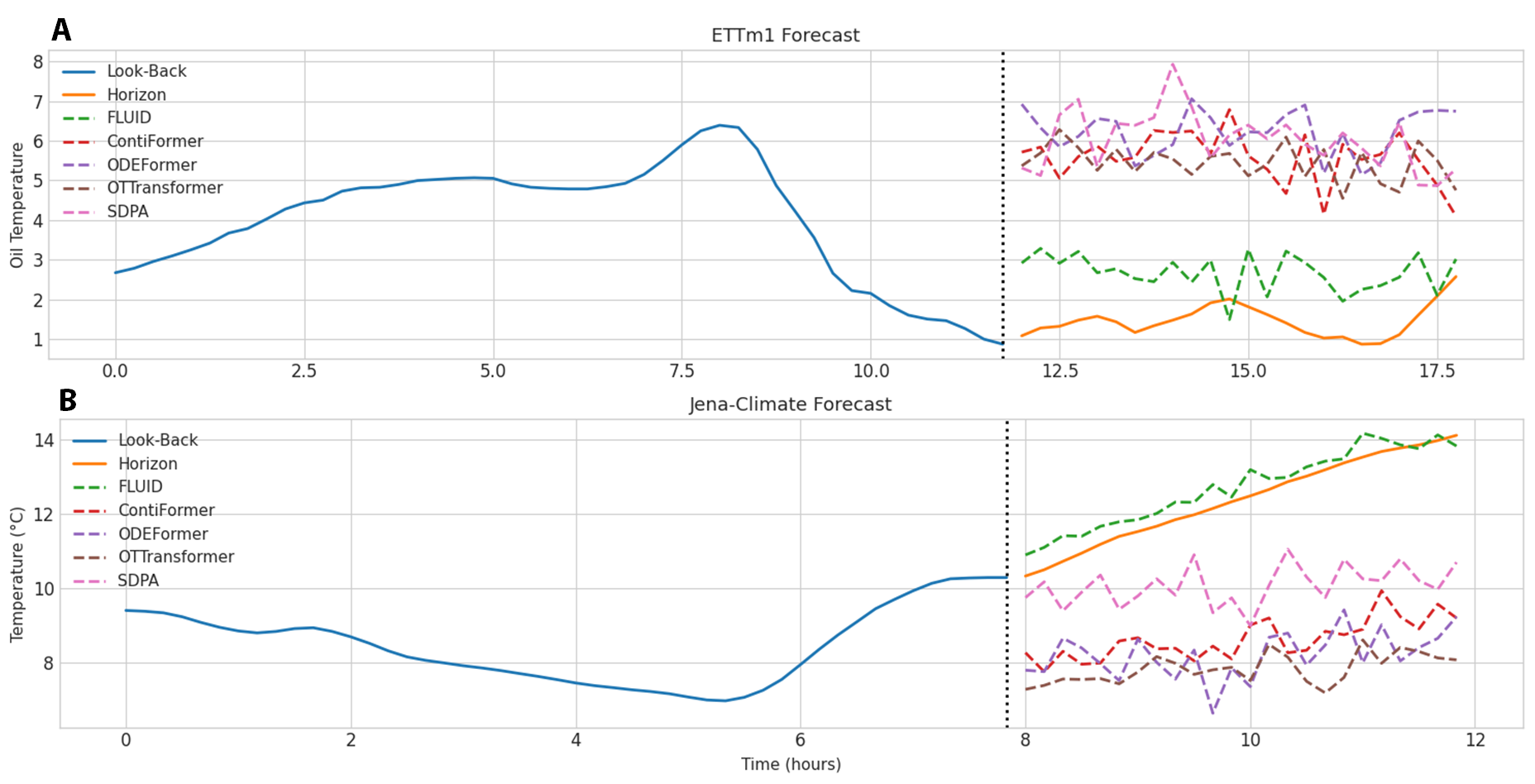}
\caption{Visualization of forecast projections: \textbf{(A)} ETTm1; and \textbf{(B)} Jena-Climate.}
\label{fig:forecast}
\end{figure}

\subsubsection{Jena-Climate} 

The Jena-Climate dataset contains real-world atmospheric recordings collected in Jena, Germany, and sampled at 10-minute intervals. It includes 14 variables, such as temperature, pressure, and humidity. Similar to ETTm1, we split the dataset into training (80\%), testing (10\%), and validation (10\%) folds We use the past 8 hours (48 intervals) of data as input to train FLUID and forecast the subsequent 4 hours (24 intervals). Prior to training, the data are normalized using \texttt{MinmaxScaler}. Post-training, the predictions are then inverse-transformed.

Table~\ref{tab:all_results} reports MSE to assess forecasting accuracy. FLUID {\color{red!70}$\mid$Main$\mid$} achieves the lowest error of 0.0021$\pm$0.0001, representing a 47.5\% improvement over the best-performing CT-Transformr baseline ContiFormer (0.0040$\pm$0.0002). This benchmark is particularly discriminative because several architectures that remain competitive on ETTm1 fail severely here. The ODEFormer reaches 0.12107$\pm$0.0749, the CTA reaches 0.1051$\pm$0.0013, and the PDE-Attention reaches 0.1854$\pm$0.0257, revealing their inability to handle high-dimensional, rapidly evolving atmospheric dynamics. The data in Figure~\ref{fig:forecast}(B) further confirm that FLUID is the only model that remains closely aligned with the Horizon ground-truth temperature curve across the full horizon, whereas all the competing models diverge visibly. Within the ablations, the error remains uniformly stable in the range of 0.0019–0.0025 across all the configurations. This robustness is attributed to the LAN’s forward-invariant logit dynamics (Theorem~\ref{theorem:state_stability}) and the adaptive step-size clamping imposed by Lemma~\ref{lemma:euler_stability}, which together prevent numerical divergence under stiff, rapidly varying inputs.

\subsection{Autonomous Vehicles (AVs) Lane-keeping Control}
In the third experiment, we evaluate FLUID as a control module for autonomous vehicles in lane-keeping tasks. The main goal is to assess its ability to function as a step-wise controller while capturing the structural relationship between the road horizon and corresponding steering commands. Experiments are conducted in two widely used environments: (i) Udacity Simulator~\cite{udacity} and (ii) OpenAI CarRacing~\cite{brockman2016openai}.

\subsubsection{Udacity Simulator} 

In the Udacity simulator, steering control is formulated as a continuous regression problem. A stream of steering angles is predicted end-to-end from visual inputs. To construct the training dataset, we manually drove the vehicle for approximately 50 minutes while recording synchronized data from the front-facing cameras (left, center, and right resulting in a dataset of approximately 15647 images each of size 320$\times$120$\times$3 paired with corresponding steering angles. Prior to training, images were preprocessed and compressed to reduce training computational overhead. The processed inputs were then fed into an end-to-end neural network consisting of convolutional feature extractors followed by the FLUID module. We adopt the image processing and convolutional heads from~\cite{shibuya_car_behavioral_cloning} and replace the dense layers with FLUID/baseline for fair comparison. 

Table~\ref{tab:all_results} reports the MSE for steering regression accuracy. FLUID {\color{red!70}$\mid$Main$\mid$} achieves a competitive MSE of 0.0174 $\pm$ 0.0018, ranking among the top-performing models alongside FLUID {\color{gray!70}$\mid$\(\mathcal{HC}_{\text{liquid}}\)$\times$2$\mid$} (0.0175 $\pm$ 0.0004) and Performer (0.0179 $\pm$ 0.0001). DeepState attains the best overall CT-baseline performance with an MSE of 0.0181 $\pm$ 0.0006, trailing FLUID by approximately 3.9\%. GRU-ODE (0.0197) and CfC-FC (0.0202) achieve moderate performance, while LTC-FC degrades to 0.0252. CT-Transformer baselines, including mTAN (0.0185) and ContiFormer (0.0188), remain competitive but consistently underperformed FLUID variants. Within ablations, smaller $\mathcal{HC}_{\text{liquid}}$ yield stronger performance. In particular, FLUID {\color{gray!70}$\mid$\(\mathcal{HC}_{\text{liquid}}\)$\times$2$\mid$} closely matches FLUID {\color{red!70}$\mid$Main$\mid$}, suggesting that reduced expansion rates provide a regularization effect in visual-motor control settings.

\begin{figure}[ht!]
\centering
\includegraphics[width=0.95\textwidth]{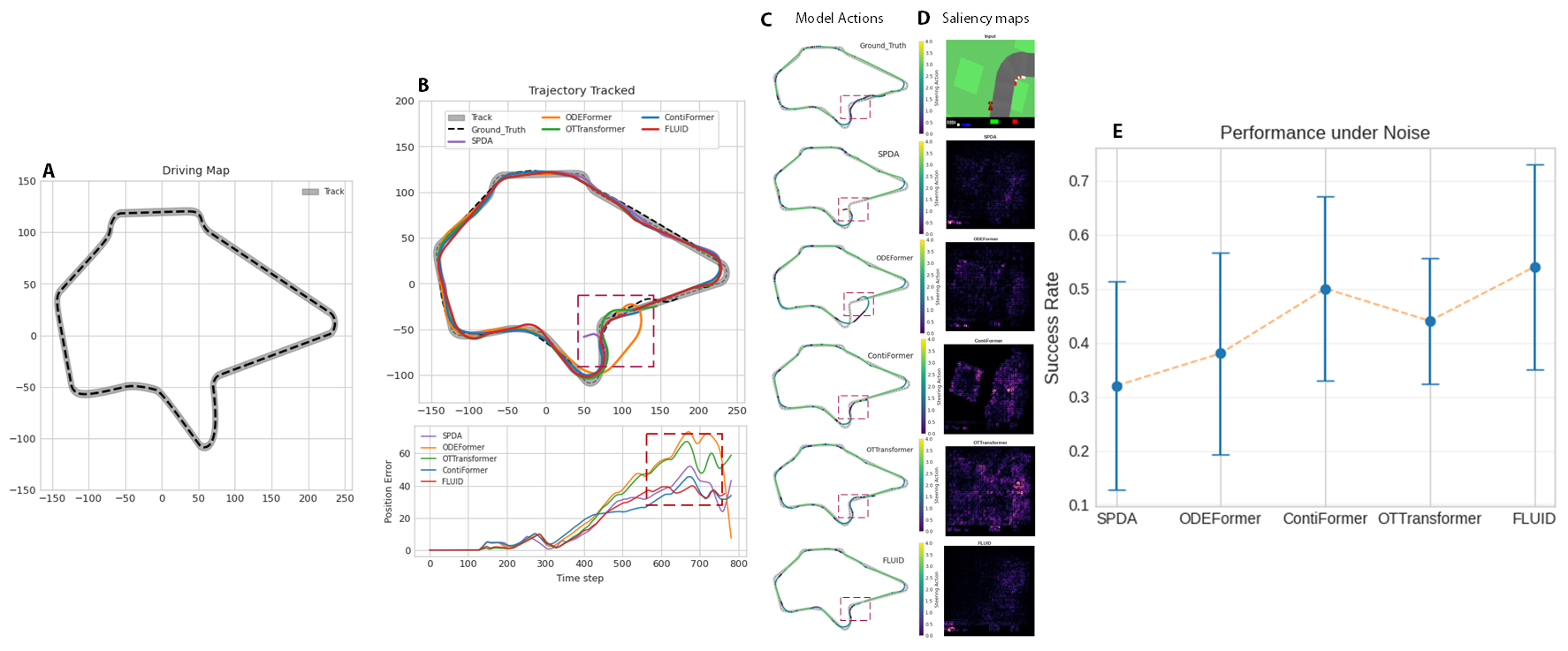}
\caption{Closed-loop analysis of OpenAI-CarRacing: \textbf{(A)} Driving map; \textbf{(B)} Trajectory tracked with position error for each model; \textbf{(C)} Action taken at each step on the map; \textbf{(D)} Saliency maps at highlighted region; \textbf{(E)} Noise test.}
\label{fig:closed-loop}
\end{figure}

\subsubsection{OpenAI CarRacing} 

In OpenAI CarRacing, control is formulated as a discrete decision-making problem, where the agent selects an action from a predefined set, making it a classification task. Training data were collected using a Proximal Policy Optimization (PPO)~\cite{schulman2017proximal} agent trained for 5 M steps. The trained agent was used to generate trajectories over 50 episodes, resulting in a dataset of approximately 48174 image-action pairs. Prior to training, no image processing was applied. We utilized the end-to-end neural network architecture proposed in~\cite{razzaq2023neural} and replaced the latent recurrent component with FLUID/baseline to ensure a fair comparison.

Table~\ref{tab:all_results} reports the prediction accuracy. FLUID {\color{red!70}$\mid$Main$\mid$} achieves the highest accuracy of 80.91$\pm$0.23, exceeding the stronger baselines CT-RNN (80.80$\pm$0.27) and SDPA Transformer (80.80$\pm$0.23). The improvement is moderate but consistent, as the performance is near saturation in this test, with competitive cluster models closely clustered between 80.13 and 80.83.

\paragraph{Closed-loop analysis:}To further validate FLUID behavior beyond classification accuracy, we design a closed-loop evaluation analysis where each model directly controls the vehicle for the full lap shown in Figure~\ref{fig:closed-loop}(A), and its trajectory is recorded from qualitative analysis. We restrict this analysis to Transformer-based models: (i) SDPA Transformer; (ii) ODEFormer; (iii) ContiFormer; and (iv) OT-Transformer to isolate the effect of structural changes on decision-making.

\textit{Trajectory Analysis:} As shown in Figure~\ref{fig:closed-loop}(B), the SDPA Transformer and ODEFormer fail to complete the track, diverging it in the highlighted part. We attribute this behavior to the attention sink, where both models disproportionately concentrate attention mass on early parts of the track where the position error is near zero, causing them to lose temporal context at critical decision points later in the track. OT-Transformer completes the lap but exhibits noticeable instability at the sharp turn in the highlighted region. ContiFormer tracks the ground truth closely and remains competitive with FLUID throughout. However, at the critical sharp turn, it drifts slightly off the road. FLUID maintains the lowest cumulative position error and remains closely aligned with the ground-truth trajectory across the entire lap.

\textit{Model Actions:} To investigate model behavior in the highlighted region, we perform a detailed analysis of the steering actions in Figure~\ref{fig:closed-loop}(C). Notably, the ground-truth trajectory in this region is itself noisy and does not provide a reliable reference, making it a particularly demanding test of a model's internal inductive biases. Despite this, FLUID consistently steers toward the road centerline through the highlighted region, suggesting that its CT dynamics induce a self-correcting inductive bias that is not present in counterparts.

\textit{Saliency Maps:} To probe what information the models rely upon when making steering decisions in this critical region, we visualize saliency maps using GradCam~\cite{selvaraju2017grad} of the convolutional feature heads in Figure~\ref{fig:closed-loop}(D). The results reveal a striking divergence in the attended regions across the models. The salience of the FLUID is concentrated almost exclusively on the road’s horizon across models with minimal activation elsewhere in the image. ContiFormer attends to both the road and other regions, which may explain its light deviation at a sharp turn. ODEFormer largely ignores the road surface entirely and bases its decision predominantly on other scene elements, which is consistent with its failure to complete the lap. OT-Transformer and SDPA Transformer exhibit diffuse, unfocused saliency patterns, suggesting that their attention mechanisms fail to extract task-relevant spatial features in this scenario.

\textit{Noise Test:} To assess robustness, we exploit OpenAI CarRacing’s built-in color-palette randomization, which replaces the appearance of the standard environment while keeping track geometry fixed. Each model is evaluated over 10 runs of 10 episodes, with the mean success rate illustrated in Figure~\ref{fig:closed-loop}(E). FLUID achieves the highest average success rate, followed by ContiFormer and OT-Transformer, while SDPA and ODEFormer collapse to below 40\%. We attribute the noise resilience of FLUID to shared input embeddings and LAN forward invariant dynamics (Theorem~\ref{theorem:state_stability}), which impose an inherent regularization effect against visual perturbations.

\subsection{Learning Physical Dynamics under Scarce Data}
In the fourth experiment, we evaluate FLUID’s ability to model physical dynamics under scarce data conditions using the well-established task of bearing degradation estimation. The task requires learning degradation dynamics and estimating the remaining useful life (RUL) as systems transition from a healthy state to failure.

Bearing degradation estimation is a classical problem in industrial engineering concerned with predicting the RUL of rolling-element bearings (REBs) from observed vibrational signals. Existing approaches are largely data-driven~\cite{razzaq2025carle,zhang2026remaining, wang2026remaining} and often require large amounts of labeled data, which are typically unavailable in real-world industrial scenarios. Moreover, these methods often exhibit limited generalization to unseen operating conditions and lack explicit physical interpretability, which restricts their practical applicability.

\textit{Objective:} The aim is to assess whether FLUID can learn a physically interpretable representation of long-range nonlinear degradation trajectories while maintaining strong generalization. In particular, we investigate its ability to learn from a limited number of samples within a single dataset and generalize to unseen datasets in zero-shot manner, capturing consistent degradation dynamics across different operating conditions and data sources.

\textit{Scarce Training Protocol:} We use three widely adopted bearing degradation datasets: (i) XJTU-SY~\cite{wang2018xjtu}, (ii) HUST~\cite{thuan2023hust}, and (iii) PRONOSTIA~\cite{nectoux2012pronostia}. Together, these datasets comprise 38 run-to-failure bearing trajectories across nine different operating conditions. To ensure a controlled evaluation setting, training is restricted to Bearings 1–3 from the first operating condition of the XJTU-SY dataset, corresponding to approximately 3.81\% of the total available samples. The learned parameters is cross-validated in zero-shot manner to HUST and PRONOSTIA dataset.

\textit{Physics-based Degradation model:} We adopted a unified physics-based degradation model to characterize the nonlinear degradation trajectory. The degradation process is modeled as a superposition of fatigue ($D_F$), wear ($D_W$), and lubrication-induced damage ($D_O$), coupled through a time-varying effective load-carrying capacity.
\begin{equation}
\frac{dD_{\text{coupled}}}{dt} = \frac{dD_F}{dt} + \gamma_w \frac{dD_W}{dt} + \zeta_L \frac{dD_O}{dt}.
\label{eq:total_degradation}
\end{equation}
Here, $D_{\text{coupled}}$ denotes the total degradation index. The parameters $\gamma_w$ and $\zeta_L$ weight the contributions of wear and lubrication to the overall degradation rate. This formulation captures complex nonlinear degradation dynamics and characterizes three stages of bearing life: (i) healthy, (ii) fault progression, and (iii) severe fault. These stages are illustrated in Figure~\ref{fig:pgnn}. The detailed derivation is provided in in Appendix~\ref{appendix:degradation_equation}.

\textit{Evaluation Metric:} Score is an evaluation metric designed for RUL estimation in the IEEE PHM challenge \cite{nectoux2012pronostia}. It is asymmetric, penalizing overestimation more heavily than underestimation. This reflects practical considerations: late maintenance predictions can result in unexpected failures with severe consequences, whereas early interventions are generally less costly.
\begin{equation}
\textit{Score} = \sum_{i : \hat{y}_i < y_i} \left( e^{-\frac{\hat{y}_i - y_i}{13}} - 1 \right) + \sum_{i : \hat{y}_i \geq y_i} \left( e^{\frac{\hat{y}_i - y_i}{10}} - 1 \right)
\label{eq:score}
\end{equation}
The metric aligns with the degradation trajectory of Eqn.~\ref{eq:total_degradation} by emphasizing errors near failure. Lower score indicate more accurate and risk-aware predictions near critical failure.


\subsubsection{Experimental Setup}
The experimental setup consists of three hierarchical steps: (i) Data preprocessing; (ii) Physics-constrained Model; (iii) Generalization tests. Figure~\ref{fig:pgnn} illustrates these steps.

\textbf{Data Preprocessing:} Degradation signals are typically consists of 1-D non-stationary vibration signals that must be preprocessed to extract meaningful patterns. We preprocess the vibration signals using the pipeline proposed in~\cite{razzaq2025carle}. Each signal is segmented into sliding windows of length~$\omega_s$ (1000 in our implementation). For each window, we apply a Morlet wavelet~\cite{lin2000feature} transform and compute wavelet coefficients. These coefficients are summarized into a set of degradation features: (i) energy~$E$; (ii) dominant frequency~$f_d$,; (iii) entropy~$h$; (iv) kurtosis~$K$; (v) skewness~$sk$; (vi) mean~$\mu$; and (vii) sample standard deviation~$\sigma$. In addition, we include the time per sample~$t_{\text{sample}}$ and the corresponding temperature~$T_{\text{sample}}$. The resulting time-frequency representation (TFR) input vector can be written as
\begin{equation}
I_{\text{TFR}} =
\Bigl[
\bigl(E, f_d, h, K, sk, \mu, \sigma\bigr)^{\text{per sensor}},
\; t_{\text{sample}},
\; T_{\text{sample}}
\Bigr]
\end{equation}

\begin{wrapfigure}{r}{0.5\columnwidth}
\vspace{-1mm}
\centering
\includegraphics[width=0.49\columnwidth]{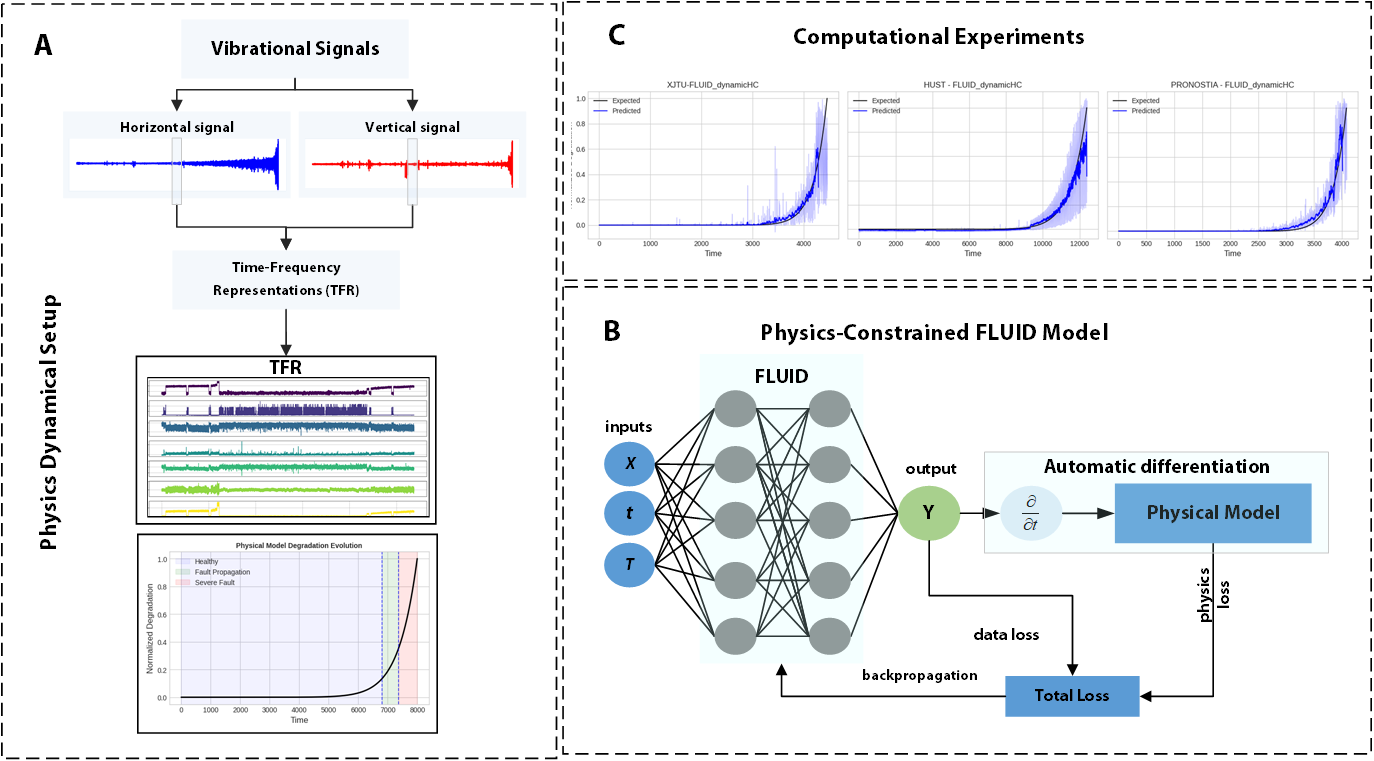}
\caption{Experimental setup for learning physical dynamics. \textbf{(A)} Raw vibration signals are converted to time-frequency representation (TFR); \textbf{(B)} Physics-constrained model is trained of TFR; \textbf{(C)} Perform cross-dataset generalization tests.}
\label{fig:pgnn}
\vspace{-3pt}
\end{wrapfigure}

\textbf{Physics-Constrained Model (PCM):} Data-driven models are often black box that lack physical inductive biases, which can lead to physically inconsistent predictions. To address we incorporate physics-based constraints into the learning objective. The total PC loss is defined as 
\begin{equation}
\mathcal{L}_{\text{total}} = \mathcal{L}_{\text{data}} + \mathcal{L}_{\text{phy}},
\label{eq:pc_loss}
\end{equation}
where $\mathcal{L}_{\text{data}}$ measures the prediction error with respect to ground truth, and $\mathcal{L}_{\text{phy}}$ penalizes violations from Eqn.~\ref{eq:total_degradation}. A custom training step to train this network is implemented based on proposed in ~\cite{razzaq2025developingdistanceawareuncertaintyquantification}.

\textbf{Computational Experiments:} From scarce samples, we extract non-overlapping time windows as individual samples, pool all samples, and randomly shuffle them. The pooled set is then partitioned into 5-folds at the sample level for cross-validation, mitigating temporal leakage. For each fold, separate models including all baselines are trained end-to-end using backpropagation through time (BPTT)~\cite{lecun1988theoretical} with the PC loss (Eqn.~\ref{eq:pc_loss}), resulting in five distinct sets of learned parameters. We report the mean ($\mu$) and standard deviation ($\sigma$) Score across folds. For qualitative interpretability, we focus on Bearing~2 across first operating conditions of all datasets, and the results are reported in Table~\ref{tab:all_results}. We also illustrate the generalization of Transformer-based models in Figure~\ref{fig:degradation_results}.



\begin{figure}[ht!]
\centering
\includegraphics[width=0.95\textwidth]{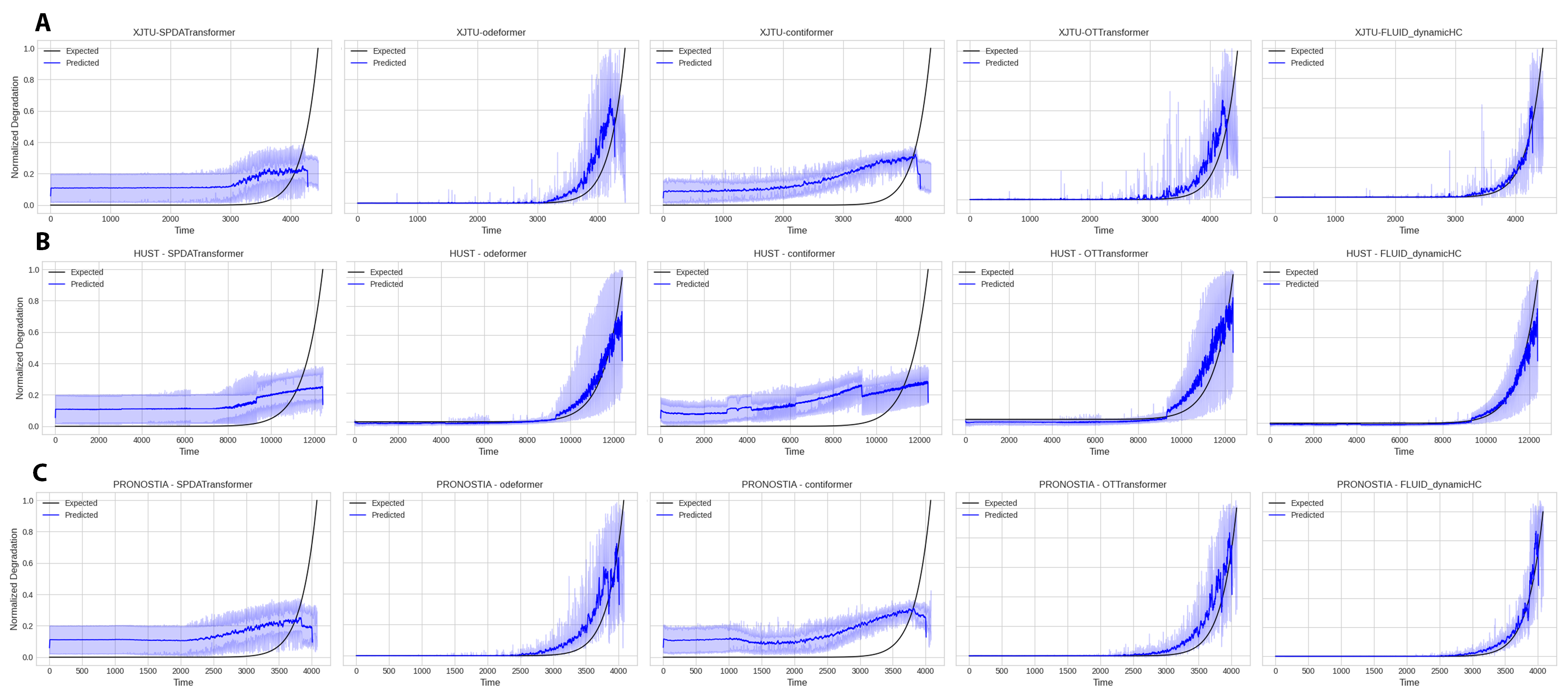}
\caption{Visualization of learned long-range nonlinear degradation trajectories by each model. \textbf{(A)} XJTU-SY dataset; \textbf{(B)} HUST dataset; \textbf{(C)} PRONOSTIA dataset.}
\label{fig:degradation_results}
\end{figure}

\subsubsection{Computational Results}

\textit{XJTU-SY:}  FLUID {\color{red!70}$\mid$Main$\mid$} achieves the lowest score of 21.29$\pm$8.94, outperforming the strongest baseline, PhasedLSTM(23.19$\pm$8.08) by 8.2\% . SDPA records 121.32$\pm$10.59, while ContiFormer (58.24$\pm$8.1) fails to learn meaningful patterns due to the scarce data regime. Removing $G_{\text{sink}}$ increases the score from 21.29 to 24.19, corresponding to a 13.6\% degradation, which directly validates the effect of attention sink. As shown in Figure~\ref{fig:degradation_results}(A), FLUID captures the full three-stage degradation curve, whereas SDPA and ContiFormer either saturate prematurely or diverge near failure.

\textit{HUST:} FLUID {\color{gray!70}$\mid$\(\mathcal{HC}_{\text{liquid}}\)$\times$8$\mid$} achieves the lowest score of 34.71 $\pm$ 5.07, followed by FLUID {\color{red!70}$\mid$Main$\mid$} at 35.67$\pm$9.05, slightly outperforming OT-Transformer (35.81$\pm$7.53).  LAN's forward-invariant dynamics and adaptive step-size clamping prevent the numerical instabilities responsible for erratic zero-shot trajectories. Figure~\ref{fig:degradation_results}(B) shows zero-shot degradation predictions on the HUST dataset and further confirms that FLUID and OT-Transformer predictions closely track the expected degradation envelope, whereas SDPA and ContiFormer exhibit flat or oscillatory trajectories that fail to capture the accelerating wear signal.

\textit{PRONOSTIA:} FLUID {\color{gray!70}$\mid$Top-\emph{K}$\times$4$\mid$} achieves the lowest aggregate score of 109.39 $\pm$ 7.18, followed by FLUID {\color{gray!70}$\mid$\(\mathcal{HC}_{\text{static}}\)$\times$4$\mid$} (109.64 $\pm$ 10.86) and FLUID~|Main| (111.54 $\pm$ 7.43). CTA is the strongest baseline at 118.61 $\pm$ 22.58, giving FLUID a 6.1\% improvement alongside a threefold reduction in variance, indicating more consistent cross-dataset generalization.  The superior performance of Top-\emph{K} sparsity on PRONOSTIA, where FLUID {\color{gray!70}$\mid$\(\mathcal{HC}_{\text{static}}\)$\times$4$\mid$} outperforms FLUID {\color{gray!70}$\mid \text{PW}\mid$} by 9.6\%, suggests that selective key attention is particularly effective when degradation trajectories are sparse and temporally localized. As illustrated in Figure~\ref{fig:degradation_results}(C), FLUID more accurately tracks the steep end-of-life regime, whereas competing baselines either underestimate the late-stage rise or produce high-variance predictions spanning the full degradation range.

\begin{wrapfigure}{r}{0.5\columnwidth}
\vspace{-13mm}
\centering
\includegraphics[width=0.49\columnwidth]{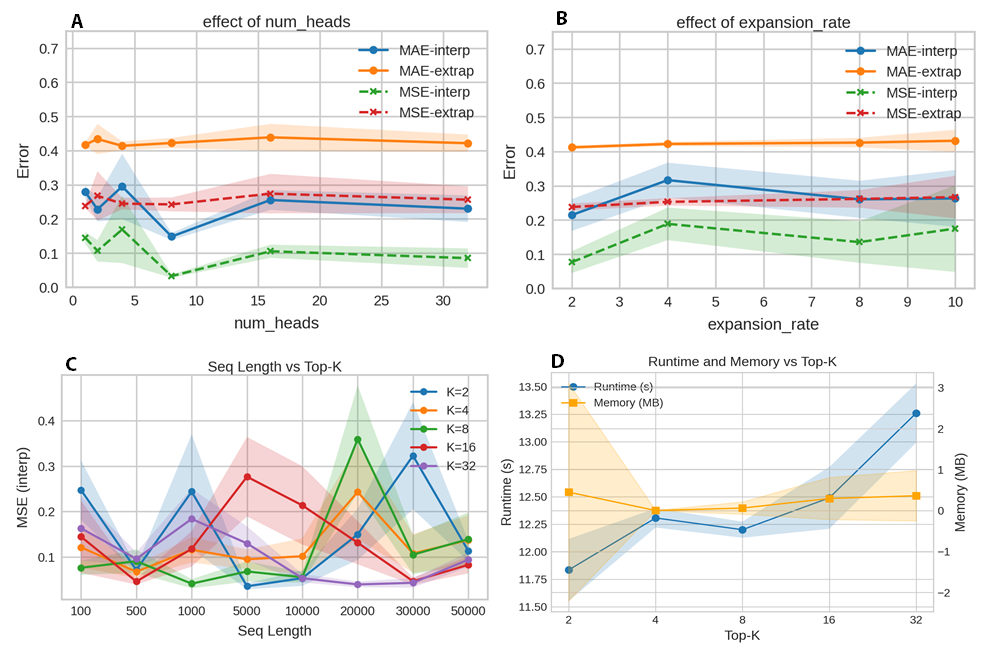}
\caption{Hyperparameter sensitivity analysis. \textbf{(A)} Effect of attention heads; \textbf{(B)} Effect of \(\mathcal{HC}_{\text{liquid}}\) expansion rate; \textbf{(C)} Top-\emph{K} vs. Sequence lengths; \textbf{(D)} Top-\emph{K} vs. run-time and memory requirements}.
\label{fig:key_hyp}
\vspace{-13pt}
\end{wrapfigure}

\subsection{Hyperparameter Sensitivity Analysis}
Sparse Top-\emph{K} attention may discard relevant contextual information from longer sequences, potentially degrading overall accuracy~\cite{tay2022efficient}. To investigate this concern alongside the influence of other key architectural hyperparameters, we conduct hyperparameter sensitivity analysis using a modified version of the irregular spiral. We examine the effects of three key hyperparameters: (i) the number of attention heads (1 to 32); (ii) the \(\mathcal{HC}_{\text{liquid}}\) expansion rate (2 to 10); and (iii) the Top-\emph{K} selection (2 to 32), evaluated across sequence lengths varying from 100 to 50,000. For the first two analyses, we report both the reconstruction MSE and the MAE under interpolation and extrapolation regimes. For the third, we report the reconstruction MSE as a function of sequence length and run-time analysis. Each configuration is evaluated over five independently initialized runs, and the mean and standard deviation are shown in Figure~\ref{fig:key_hyp}.

\textbf{Effect of attention heads:} The effect of increasing attention heads is illustrated in Figure~\ref{fig:key_hyp}(A). The relationship between the number of attention heads and reconstruction error is non-monotonic. Increasing the number of heads from 1 to 2 yields a marginal improvement in error; however, further increases beyond this point tend to degrade performance, suggesting that excessive head partitioning fragments the representational capacity available to each head. A configuration of 8 attention heads achieves the lowest error in both interpolation and extrapolation settings, after which the performance stabilizes as the sensitivity further increases.

\textbf{Effect of Expansion Rate:} The effect of increasing expansion rate is illustrated in Figure~\ref{fig:key_hyp}(B). The expansion rate of the \(\mathcal{HC}_{\text{liquid}}\) also has a non monotonic effect on the reconstruction error. The lowest error is achieved at an expansion rate of 2, with performance gradually deteriorating as the rate increases beyond this value. This suggests that excessive within the hyper-hidden pace introduces redundancy and may impair the model’s ability to regulate interlayer signal flow effectively. These results motivate careful selection of the expansion rate, with smaller values generally preferred.

\textbf{Effect of Top-\emph{K} vs. Sequence Length:} The effect of Top-\emph{K} selection is illustrated in Figure~\ref{fig:key_hyp}(C) with memory and run-time in Figure~\ref{fig:key_hyp}(D). The interaction between Top-\emph{K} selection and sequence length reveals a subtle trade-off between accuracy and computational efficiency. Across most sequence lengths, values in the range 8$\leq$K$\leq$16 provide the most favorable balance, achieving a lower interpolation MSE while maintaining traceable run-time and memory requirements. With respect to shorter sequences, smaller \emph{K} values introduce greater variance, reflecting increased sensitivity to the selection of informative query-key pairs. With longer sequences ($\geq$20,000), \emph{K}=32 achieves the lowest error, suggesting that denser key coverage becomes increasingly beneficial as temporal complexity increases. These findings support the use of \emph{K}=8 as a practical default, with \emph{K}=16 or \emph{K}=32 recommended for tasks involving very long sequences where computational resources permit.

\vfill
\pagebreak

\begin{wrapfigure}{r}{0.50\textwidth}
\vspace{-13mm}
\centering
\captionsetup{type=table}
\caption{Run-Time and Memory Results}
\label{tab:run_time}
\resizebox{0.49\textwidth}{!}{%
\begin{tabular}{lccc}
\toprule
\textbf{Model} & \makecell{\textbf{Run-Time}\\(s)} & \makecell{\textbf{Throughput}\\(seq/s)} & \makecell{\textbf{Peak Memory}\\(MB)} \\
\midrule
CT-RNN
& 7.2746\textsuperscript{\scriptsize ±0.3048} & 0.14 & 0.38 \\
GRU-ODE
& 12.5673\textsuperscript{\scriptsize ±0.1260} & 0.08 & 0.35 \\
PhasedLSTM
& 5.1821\textsuperscript{\scriptsize ±0.2735} & 0.19 & 0.52 \\
LTC-FC
& 14.9679\textsuperscript{\scriptsize ±0.2290} & 0.07 & 0.84 \\
LTC-NCP
& 15.7190\textsuperscript{\scriptsize ±0.4297} & 0.06 & 0.97 \\
CfC-FC
& 6.2448\textsuperscript{\scriptsize ±0.2805} & 0.16 & 1.04 \\
CfC-NCP
& 14.239\textsuperscript{\scriptsize ±0.3327} & 0.05 & 0.57 \\
S4
& 0.0233\textsuperscript{\scriptsize ±0.0036} & 42.94 & 66.83 \\
SDPA-Transformer & 0.0195\textsuperscript{\scriptsize ±0.0036} & 51.40 & 321.36 \\
mTAN
& 0.0292\textsuperscript{\scriptsize ±0.0044} & 34.27 & 789.83 \\
ODEFormer
& 0.0338\textsuperscript{\scriptsize ±0.0009} & 29.60 & 81.95 \\
ContiFormer
& 0.0545\textsuperscript{\scriptsize ±0.0038} & 18.34 & 81.73 \\
CTA
& 8.8892\textsuperscript{\scriptsize ±0.2962} & 0.11 & 0.92 \\
OT-Transformer
& 0.0958\textsuperscript{\scriptsize ±0.0029} & 10.43 & 259.08 \\
PDE-Attention
& 0.0570 \textsuperscript{\scriptsize ±0.0166} & 17.54 & 658.17 \\
\midrule
FLUID {\color{red!70}$\mid$Main$\mid$}
& 0.2086\textsuperscript{\scriptsize ±0.1030} & 4.79 & 124.95  \\
\bottomrule
\end{tabular}%
}
\vspace{-10pt}
\end{wrapfigure}

\subsection{Scalability and Efficiency of FLUID}

We evaluate the efficiency and scalability of FLUID by benchmarking run-time, throughput, and peak memory usage. For this experiment, we fix the model dimension/units to 64, the number of attention heads to 4, the batch size to 1, and the sequence length to 1024 and use Google Colab T4 GPU to conduct this test. For each model, we perform 10 forward passes and report the mean and standard deviation of run-time, throughput and peak memory consumption, in Table~\ref{tab:run_time}. FLUID achieves an average inference time of 0.2086 $\pm$ 0.1030 s and a throughput of 4.79 seq/s, placing it in a competitive mid-range regime. It is substantially faster than CT-RNNs such as GRU-ODE (12.57 s) and LTC-FC (14.97 s), which suffer from inherently sequential computation. Compared with CT Transformer variants, FLUID remains efficient while avoiding the higher computational cost observed in models such as CTA (8.89 s). Although methods such as OT-Transformer (0.0958 s) and PDE-Attention (0.0570 s) achieve lower latency, they do so at the expense of significantly higher memory usage. In terms of memory, FLUID requires 124.95 MB, which is considerably lower than those of the SDPA-Transformer (321.36 MB), mTAN (789.83 MB), OT-Transformer (259.08 MB), and PDE-Attention (658.17 MB). These findings demonstrate that compared with both the standard and CT-Transformer baselines, FLUID provides strong memory efficiency.

\section{Discussion}\label{sec:discuss}
While FLUID demonstrates promising gains, several aspects can be further extended:

\textbf{Stochastic LAN:} Replacing the LAN-ODE with  stochastic differential equation (SDE) is a natural direction for future work, as it will introduce intrinsic uncertainty quantification into the attention mechanism.

\textbf{Hyper-Connections:} Unconstrained hyper-connections in residual mappings may disrupt identity propagation, leading to signal amplification or attenuation that destabilizes large-scale training. Exploring manifold-constrained hyper-connections (mHC)~\cite{xie2025mhc} offers a potential direction for regulating signal flow and improving training stability. 

\textbf{Learnable sparse Top-\emph{K} selection:} Developing an adaptive or learnable Top-\emph{K} selection strategy may improve key scoring, and incorporating hardware-aware optimization could further enhance accuracy and robustness in future work.

\section{Conclusion}\label{sec:conclude}
We present FLUID, a novel continuous-time Transformer architecture that embeds continuous dynamics directly into attention computation via the Liquid Attention Network (LAN). LAN rethinks attention logits as the solution to a first-order linear ODE modulated by input-dependent nonlinear interlinked recurrent gates. Theoretically, we establish stability guarantees for LAN dynamics and prove that the LAN serves as an interpolating middle ground between SDPA and CT-RNNs recovering each as a special case under a well-defined parameterization of its gating functions. The LAN is augmented by an explicit attention-sink gate that eliminates disproportionate mass on uninformative nodes. To resolve the seesaw effect in standard residual connections, FLUID replaces them with Liquid Hyper-Connection that dynamically regulates interlayer information flow. We evaluated FLUID across a broad set of learning problems: (i) irregular time-series modeling; (ii) long-range forecasting; (iii) lane-keeping control of autonomous vehicles; and (iv) physical dynamics learning under scarce data regimes. FLUID consistently matches or outperforms competing baselines, achieving up to 47\% improvement while demonstrating markedly superior generalization under distributional shift and low-data regimes. In autonomous driving tasks, FLUID exhibits strong noise resilience and a self-correcting inductive bias, maintaining the lowest cumulative position error and robust closed-loop performance even under visual randomization. Ablation studies confirm that both LAN and Liquid Hyper-Connections deliver complementary gains in stability, expressivity, and robustness. Finally, we analyze the impact of key hyperparameters and discuss scalability and efficiency. FLUID occupies a practical middle ground between CT-RNNs and Transformers in both runtime and memory usage, offering a favorable trade-off between modeling flexibility and computational cost.

\subsection{Broader Impact} 
This work rethinks SDPA by modeling attention logits as CT dynamics governed by linear ODEs with nonlinear, interlinked gates. Reformulating attention as an evolving process can enable more interpretable explanations and can pave the way for biologically plausible attention mechanisms i.e., Neuronal Attention Circuit (NAC)~\cite{razzaq2025neuronal} by authors. FLUID transformer offers a general alternative to standard Transformers for accurate multivariate time-series modeling.


\section*{Data \& Code Availability}
The data and code for reproducibility are available at \url{https://github.com/itxwaleedrazzaq/fluid-transformer}.




\bibliographystyle{unsrt}
\bibliography{references}

\appendix

\section{Preliminaries}
In this section, we will provide a comprehensive background.
\subsection{SDPA Transformer}\label{appendix:sdpa}
The SDPA Transformer is a sequence modeling neural network architecture based on an SDPA mechanism, enabling parallel computation to model long-range dependencies effectively. It consists of three components: (i) input embeddings, (ii) encoder, and (iii) decoder.

\textbf{Input Embeddings: } Each input token is mapped to a continuous vector, and positional information is injected to retain order. Given token embeddings $E \in \mathbb{R^{(n\times d)}}$, positional encoding $P \in \mathbb{R^{(n\times d)}}$ are added to obtain the input representation: \(X = E + P\). A common choice is sinusoidal positional encoding:
\begin{align}
\mathrm{PE}(pos, 2i) &= \sin\left(\frac{pos}{10000^{2i/d}}\right), \quad
\mathrm{PE}(pos, 2i+1) = \cos\left(\frac{pos}{10000^{2i/d}}\right)
\end{align}

\textbf{Encoder:} The encoder consists of a stack of identical blocks, each block contains a multi-head self-SDPA module followed by a position-wise feed-forward network (FFN) with residual connection and layer normalization applied after each sub-layer. Self-SDPA allows each position to attend to all tokens in the sequence. Given input $X$, a encoder block computes:
\begin{align}
H &= \mathrm{LayerNorm}(X + \mathrm{MultiHead}(X, X, X)) \\
\mathrm{Enc}(X) &= \mathrm{LayerNorm}(H + \mathrm{FFN}(H))
\end{align}

\textbf{Decoder:} The decoder also consists of a stack of identical blocks, each consisting of three submodules: (i) masked self-SDPA, cross-SDPA, and an FFN. Masked Self-SDPA restricts each position to attend only to previous positions, ensuring autoregressive generation, while cross-SDPA conditions on the encoder representations. For input $Y$ and encoder output $Z$, a decoder block computes:
\begin{align}
H_1 &= \mathrm{LayerNorm}(Y + \mathrm{MultiHead}(Y, Y, Y)) \\
H_2 &= \mathrm{LayerNorm}(H_1 + \mathrm{MultiHead}(H_1, Z, Z)) \\
\mathrm{Dec}(Y, Z) &= \mathrm{LayerNorm}(H_2 + \mathrm{FFN}(H_2))
\end{align}

\subsubsection{SDPA Mechanism}\label{appendix:prelim_attention}
SDPA mechanisms have become a cornerstone in modern neural architectures, enabling models to dynamically focus on relevant parts of the input. The concept was first introduced in the context of neural machine translation, where it allowed the decoder to weight encoder outputs according to their importance for generating each target token. Formally, given a query vector \(q \in \mathbb{R}^d\), key vectors \(K = [k_1, k_2, \dots, k_n] \in \mathbb{R}^{n \times d}\), and value vectors \(V = [v_1, v_2, \dots, v_n] \in \mathbb{R}^{n \times d}\), the attention mechanism can be expressed in two steps:

\begin{enumerate}
    \item Compute the scaled-dot-attention logits:
    \begin{equation}
    a_i = \frac{q^\top k_i}{\sqrt{d}}
    \label{eq:attention_logit}
    \end{equation}
    
    \item Normalize the logits to get attention weights and compute the output:
    \begin{equation}
    \alpha_i = \text{softmax}(a_i) =  \frac{e^{a_i}}{\sum_{j=1}^n e^{a_j}}, \quad \text{SDPA}(Q, K,V) = \sum_{i=1}^n \alpha_i v_i
    \end{equation}
\end{enumerate}

Here, \(a_i\) is the raw attention logit between the query and each key, and the scaling factor \(\sqrt{d}\) prevents large dot products from destabilizing the softmax \cite{vaswani2017attention}.

\subsection{Liquid Neural Networks (LNNs)}\label{appendix:prelim_lnn}

Liquid Neural Networks (LNNs) are a class of CT-RNNs that represent a dynamical system with variable time constants associated with their hidden states. Unlike standard CT-RNNs, which define the system’s dynamics through implicit nonlinearities, LNNs use a first-order linear dynamical system coupled with nonlinear, interlinked gates. The dynamics of the hidden state $\mathbf{x_t}$ are described as:
\begin{equation}
\frac{d\mathbf{x_t}}{dt} = -\frac{\mathbf{x_t}}{\tau} + \mathbf{S_t},
\label{eq:raw_lnn}
\end{equation}
where $\mathbf{S_t}$ represents the nonlinear term, defined as \(\mathbf{S_t} = f(\mathbf{x_t}, \mathbf{u}, t, \theta) (A - \mathbf{x_t}),\) with $A$ being a learnable parameter matrix. Plugging $\mathbf{S_t}$ into the Eqn.~\ref{eq:raw_lnn} gives:
\begin{equation}
\frac{d\mathbf{x_t}}{dt} =
\underbrace{\left[\frac{1}{\tau} + f(\mathbf{x_t}, \mathbf{u}, t, \theta)\right]}_{{\omega_\tau}}
\mathbf{x_t}
+ \underbrace{f(\mathbf{x_t}, \mathbf{u}, t, \theta)A}_\phi
\label{eq:lnn}
\end{equation}
The neural network $f$ not only determines the derivative of the hidden state but also serves as an input-dependent, variable time constant, defined as:
\begin{equation}
\tau_{sys} = \frac{\tau}{1 + \tau \, f(\mathbf{x_t}, \mathbf{u}, t, \theta)}.
\end{equation}
This formulation allows LNNs to dynamically adjust the effective time constant of each hidden unit based on the current state and input, providing greater adaptivity and robustness compared to traditional CT-RNNs. LNNs are known for their strong expressivity, stability, and performance in irregularly sampled time-series modeling \cite{hasani2021liquid, hasani2022closed}.

\section{Physics-based Degradation model}~\label{appendix:degradation_equation}

A detailed rationale for the derivation is provided in~\cite{razzaq2025developing}; we briefly outline it here.

\textbf{Effective Load rating: }The effective load-carrying capacity (\(C_{\mathrm{eff}}\)) evolves by scaling the nominal rating ($C_{\text{Load}}$) according to viscosity loss ($\nu/\nu_0$), geometric degradation ($W_{\text{mod}}$), and contamination effects from debris concentration ($C_{\text{debris}}$) as
\begin{equation}
C_{\mathrm{eff}} = C_{load}\cdot \frac{\nu}{\nu_0}\cdot W_{mod}\cdot
\exp\!\bigl(-\psi C_{debris}\bigr),
\end{equation}
\textbf{Fatigue:} The Fatigue life degradation (\(D_F\)) is modeled as:
\begin{equation}
\frac{dD_F}{dt} = \left( \frac{P}{C_{\text{eff}}} \right)^{p} \frac{n}{60 \times 10^{6}} + \beta \phi \left( \frac{P}{C_{\text{eff}}} \right)^{q} n
\end{equation}
where $D_F$ evolves as a load-normalized, cycle-based term driven by radial load $P$ and speed $n$, augmented by an equivalent damaged volume (EDV) scaled by $\beta$ and $\phi$.

\textbf{Wear:} The wear degradation (\(D_W\)) evolves according to
\begin{equation}
\frac{dD_W}{dt} = \frac{A_v P}{H_{hard}}\frac{ds}{dt} + A_a R.
\end{equation}
where $D_w$ is modeled as the sum of load-driven sliding Archrad wear~\cite{kauzlarich2001archard}, which is proportional to hardness $H_{\text{hard}}$, and an abrasive component amplified by surface roughness $R$.

\textbf{Lubrication: }Lubrication and thermal degradation (\(D_O\)) are modeled by
\begin{equation}
\frac{dD_O}{dt} =\frac{1}{m c_p} \Bigl[ \mu_f P \omega - hA(T-T_a) + \xi\frac{dO}{dt} \Bigr].
\end{equation}
This equation captures $D_O$ by balancing frictional heat generation from load $P$ and angular speed $\omega$, heat dissipation to the environment, and heat released by lubricant oxidation $O$.

\textbf{Geometric: }Geometric degradation (\(W_{mod}\)) of the contact is represented by
\begin{equation}
W_{mod}
=
\bigl(1+\eta V+\zeta R^2\bigr)^{-1},
\end{equation}
where $W_{\text{mod}}$ is the geometric modifier that reduces load capacity as the accumulated wear volume $V$ and surface roughness $R$ increase.

\textbf{Stochastic contamination: }Stochastic contamination dynamics are modeled as
\begin{equation}
dC_{debris} = \rho\frac{dD_W}{dt}dt +\sigma_c d\mathcal{W}_t,
\end{equation}
The stochastic equation models debris concentration $C_{\text{debris}}$ as wear-generated particle production with random fluctuations driven by a Wiener process $\mathcal{W}_t$.


\end{document}